\documentclass{article}

\usepackage[preprint]{neurips_2026}

\usepackage[utf8]{inputenc} % allow utf-8 input
\usepackage[T1]{fontenc}    % use 8-bit T1 fonts
\usepackage{hyperref}       % hyperlinks
\usepackage{url}            % simple URL typesetting
\usepackage{booktabs}       % professional-quality tables
\usepackage{amsfonts}       % blackboard math symbols
\usepackage{nicefrac}       % compact symbols for 1/2, etc.
\usepackage{microtype}      % microtypography
\usepackage{xcolor}         % colors
\usepackage{enumitem}
\usepackage{amsmath, amsthm, amssymb}
\usepackage{mathtools}
\usepackage{graphicx, subcaption}

\theoremstyle{definition}
\newtheorem{definition}{Definition}[section]

\theoremstyle{plain}
\newtheorem{theorem}[definition]{Theorem}
\newtheorem{proposition}[definition]{Proposition}
\newtheorem{lemma}[definition]{Lemma}

\newtheorem{assumption}[definition]{Assumption}
\newtheorem{corollary}[definition]{Corollary}
\newtheorem{conjecture}[definition]{Conjecture}

\theoremstyle{remark}
\newtheorem{remark}[definition]{Remark}

\title{Grokking as Structural Inference: Transformers Need Bayesian Lottery Tickets}

\author{%
	Kai Hidajat \quad Solden Stoll \quad Joseph An \\
	Department of Computer Science \\
	University of Washington\\
	Seattle, WA 98195 \\
	\texttt{\{hidajatk, solden, anjo0\}@uw.edu} \\
}

\begin{document}

\maketitle

\begin{abstract}
	% Introduce grokking and lottery ticket hypothesis --- this is the central framing
	% Results etc.
	Why does a Transformer that has memorized its training set wait thousands of steps before it generalizes? Existing accounts locate this delay in norm minimization, feature emergence, or the late discovery of sparse subnetworks. These explanations capture important parts of the transition, but ignore a constraint unique to attention-based models: if attention discards an informative token, no bounded downstream computation can recover it. We formalize attention as an implicit Bayesian posterior over the task dependency graph and prove that generalization requires two separable conditions: a familiar Goldilocks bound on MLP capacity, coinciding with norm-based theories of grokking, and a novel Bayesian structural condition requiring attention to place sufficient mass on every informative token. This decoupling explains delayed generalization as delayed structural inference. Early in training, the MLP memorizes through unaligned features, drives the cross-entropy loss near zero, and thereby starves attention of structural gradient. Weight decay must then erode memorization before the missing graph becomes learnable, yielding the known inverse-weight-decay delay, which we derive as a structural waiting time. We then prove that this explaining-away delay can be bypassed by a KL-based structural intervention, yielding an inverse-intervention-strength scaling law for the grokking time. Experiments on algorithmic sequence tasks isolate structure from capacity and show that this Bayesian ticket matches or outperforms lottery-ticket transfer.

\end{abstract}

\section{Introduction}\label{sec:intro}
%% Grokking and current causal theories of grokking

Generalization is the central desideratum of machine learning, yet the mechanisms by which it emerges remain only partially understood. Grokking \citep{Power_Burda_etal_2022} makes this tension unusually visible. A network first memorizes its training set, then, after many further epochs, abruptly discovers the rule it seemed to have missed. Once a curiosity of modular arithmetic, the phenomenon has now been observed in sparse parity, group composition, and more realistic settings with natural data \citep{Liu_Michaud_etal_2023,Golechha_2024}. It therefore offers a rare empirical probe for understanding what changes when a model ceases merely to fit and begins to understand.

% The same tension between memorization and generalization arises in large-scale language model training \citep{Nanda_Chan_etal_2023}. The mechanisms that determine whether (or how quickly) a network transitions to a generalizing solution have direct implications for training efficiency, interpretability, and alignment. A principled theory of grokking is therefore a principled theory of a fundamental aspect of deep learning dynamics.

%%  Norm optimization and loss landscape
%%      LU / Goldilocks -- Omnigrok
%%      Any sparse regularizer -- Beyond Euclidian norm
%%      Zero loss manifold -> Dichotomy of implicit bias + Geometry of grokking

%%  Feature learning and transition dynamics
%%      Lazy vs rich regime -- Grokking as the transition from lazy to rich
%%      Feature learning -> emergence -> decoupling -- Information theoretic progress measures
%%      Lazy -> independent feature -> interactive feature -- Li^2

%%  Structural / circuit conditions
%%      Fourier circuits -- Progress measures for mechinterp
%%      Circuit complexity -- Progress measures on real-world
%%      Lottery ticket -- Bridging Lottery ticket

Existing causal theories of grokking generally fall into three families. The first focuses on \textbf{norm optimization and the loss landscape}, arguing that weight decay or other implicit biases drive the network across a zero-loss manifold into a ``Goldilocks zone'' of parameter norms where generalization naturally occurs \citep{Liu_Michaud_etal_2023, Boursier_Pesme_etal_2025, Notsawo_Dumas_etal_2025, Lyu_Jin_etal_2024, Musat_2026}. The second focuses on \textbf{feature learning and transition dynamics}, framing grokking as a phase transition from a lazy, memorizing regime to a rich, feature-learning regime \citep{Kumar_Bordelon_etal_2024, Clauw_Stramaglia_etal_2024, Tian_2025}. The third focuses on \textbf{structural and circuit conditions}, characterizing grokking as the sudden emergence of specific ``winning tickets'' \citep{Minegishi_Iwasawa_etal_2025, Golechha_2024, Nanda_Chan_etal_2023}.

%% Limitations
We argue that all three families are incomplete in the same way: they neglect the \emph{attention distribution} as an independent locus of the generalization condition. A Transformer that spreads attention poorly across task positions cannot generalize, regardless of its MLP norm or feature richness. Generalization requires that attention concentrate specifically on the \emph{informative} positions, i.e. the positions the task's target function actually depends on.

% We formalize this intuition through a connection to \emph{Bayesian structural inference} \citep{Singh_Buckley_2023}: the attention distribution over positions is an implicit posterior over an underlying dependency graph, and generalization requires this posterior to approximate the oracle (ground-truth) dependency structure. This is a \emph{structural alignment} condition, distinct from and independent to norm control.

Our main contributions are as follows:
\begin{enumerate}[label=(\roman*)]
	\item \textbf{The Decoupling Theorem} (\S\ref{sec:decoupling}). We prove that for a single-layer Transformer trained on a structured sequence prediction task, zero test error strictly necessitates a Bayesian structural condition \(\mathcal B_\gamma\) on the attention distribution. Under decoupled max-margin assumptions, we further show that zero test error is \emph{mathematically equivalent} to the intersection of \(\mathcal{B}_\gamma\) and a norm condition on the MLP (\emph{Goldilocks zone}).

	\item \textbf{The Convergence Theorem} (\S\ref{sec:convergence}). We characterize dynamics by which a norm-contractive optimizer drives the attention distribution toward the oracle graph, split across four distinctly characterized phases. We identify the second phase, the \emph{explaining-away plateau}, during which the MLP's memorization suppresses the structural gradient to the attention parameters, as the dominant source of grokking delay. We then provide a bound on this delay in terms of the optimizer's contraction rate and the task's signal gap.

	\item \textbf{KL Acceleration} (\S\ref{sec:convergence}). We show theoretically that augmenting the training loss with a KL divergence penalty toward the oracle attention distribution eliminates the explaining-away plateau entirely. This yields exponential convergence of the attention posterior throughout training, with a theoretical speedup scaling linearly in the effective KL penalty strength.

	\item \textbf{Experimental Validation} (\S\ref{sec:experiments}). We confirm the Decoupling Theorem's predictions on a modular arithmetic benchmark. We include norm-matched experiments that successfully isolate structural alignment from norm effects. We verify the four-phase attention dynamics and demonstrate that our KL intervention can predictably accelerate or prevent grokking.
\end{enumerate}

\section{Related Works}\label{sec:related}

\paragraph{Norm, Feature, and Delay Theories.}
A dominant explanation of grokking is that optimization on the zero-loss manifold drives parameters from a high-norm memorizing solution into a norm-controlled ``Goldilocks zone'' \citep{Liu_Michaud_etal_2023, Boursier_Pesme_etal_2025, Notsawo_Dumas_etal_2025, Lyu_Jin_etal_2024, Musat_2026}. Another theory frames grokking as a transition from lazy memorization to feature learning \citep{Kumar_Bordelon_etal_2024, Clauw_Stramaglia_etal_2024, Tian_2025}. These views correctly identify the capacity and timescale constraints, i.e., weight decay must erode memorization before generalizing structure can emerge. However, we show that even if a norm-controlled MLP has useful features, it still cannot generalize if attention fails to route the informative tokens.

\paragraph{Attention as Structural Inference.}
Sparse attention is a feature-selection mechanism as well as an efficiency device \citep{Katharopoulos_Vyas_etal_2020}. Recent theory formalizes attention as Bayesian structural inference, where the attention distribution behaves like a posterior over relevant context tokens or graph dependencies \citep{Singh_Buckley_2023, Zhang_Frei_etal_2023}. We import this view into grokking: the learned posterior must place sufficient mass on the oracle dependency graph. This connects continuous attention learning to the discrete routing perspective behind program-like Transformer analyses such as RASP \citep{Weiss_Goldberg_etal_2021}.

\paragraph{Structural Circuits and Gradient Starvation.}
Mechanistic accounts identify grokking with the emergence of compact circuits or transferable winning tickets \citep{Nanda_Chan_etal_2023, Minegishi_Iwasawa_etal_2025, Frankle_Carbin_2019}. Other work observes delayed hierarchical structure \citep{Murty_Sharma_etal_2023, DuSell_Chiang_2024} and gradient starvation, where high-capacity downstream components suppress upstream learning \citep{Pezeshki_Kaba_etal_2021}. Our contribution is to make these threads explicit in the attention setting: the winning ticket decomposes into routing and representation components, and the explaining-away plateau is the period during which the MLP has solved the training labels while the attention posterior has not yet learned the task graph.

\section{Preliminaries}\label{sec:definitions}
%% Proof setup and grounding
%%  Defining task and model
%%  Ground definitions in existing sparse attention theory
To understand why a neural network suddenly generalizes after thousands of steps of overfitting, we must track how it learns to separate signal from noise. Recent high-dimensional analyses of Transformers achieve this by explicitly disentangling the network's \emph{routing} mechanism (attention) from its \emph{representation processing} (the MLP) \citep{Barnfield_Cui_etal_2025, Chen_Luo_2025, Ahn_Cheng_etal_2024}. In this section, we formalize the learning environment, the model architecture, and the optimization dynamics that drive the late-stage phase transition.

\subsection{Structured Sequence Prediction Task}\label{sec:prelim_task}

Grokking is typically observed on algorithmic datasets where the underlying rule depends strictly on a few key tokens, while the rest of the sequence acts as padding \citep{Power_Burda_etal_2022}. We formalize this using the sparse-token classification framework recently analyzed by \citet{Barnfield_Cui_etal_2025}.

Let \(\mathcal{V}\) be a finite vocabulary of discrete tokens, and let our inputs be sequences \(s \in \mathcal{V}^T\) drawn from a data distribution \(\mathcal{D}\). The model's objective is to predict a target class \(y = f^*(s) \in \mathcal{Y}\).

\begin{definition}[\(k\)-sparsity]\label{def:k-sparse}
	We say the task has \(k\)-sparse dependencies if there exists a fixed set of \(k\) \emph{informative positions} \(\mathcal{I}^* \subset \{0,\dots,T-1\}\) (where \(k \ll T\)) that entirely determine the target \(y\). Thus, all other positions \(j \notin \mathcal I^*\) are \emph{distractors}, and do not contribute to \(y\).

	For standard algorithmic tasks, the model outputs its prediction at the final position of the sequence, \(q = T-1 \notin \mathcal I^*\) (e.g., the equals sign). We restrict our theoretical analysis to this fixed query position setting for analytical tractability and consistency with common experimental setups; concrete task instantiations are listed in Appendix Table~\ref{tab:tasks}.
\end{definition}

\subsection{Attention as Structural Routing}\label{sec:prelim_model}

To render the optimization dynamics tractable while preserving the core phenomena, we analyze a minimal Transformer consisting of a single-layer, single-head attention module followed by a two-layer ReLU MLP. This abstraction is heavily utilized in recent theory: single-head linear attention preserves the essential optimization landscape of full Transformers \citep{Ahn_Cheng_etal_2024}, and the Attention-MLP stack cleanly captures the representational limits we wish to study \citep{Sanford_Hsu_etal_2023, Oymak_Rawat_etal_2023}.

%%  Instead of inventing "minimal structural sparsity," connect to known results:
%%      Existing work shows attention is naturally \(n^C\)-sparse [Chen et al., 2024]. Our "minimal structural sparsity" corresponds to the case where the sparse support must include all informative positions.
Given an input sequence \(s\), let \(e_j(s) = e(s_j) + p_j \in \mathbb{R}^d\) denote the position-aware embedding of the token at position \(j\). At the query position \(q\), the attention head computes a convex combination over the context window, yielding attention weights \(\alpha\) that reside on the probability simplex \(\Delta^{T-1}\). The restriction of attention weights to the simplex is both standard and physically grounded; recent work proves that attention naturally condenses into \(n^C\)-sparse subsets over the simplex during training \citep{Deng_Song_etal_2025, Zucchet_dAngelo_etal_2025}.

%%  Simplify "Informative Separability"
%%      Consider a weaker but sufficient condition: The subspace spanned by informative position embeddings is linearly separable from the distractor subspace under the value projection.
The attention layer acts as a structural router, pooling the sequence into a single \(d_v\)-dimensional representation \(h(s) = \sum_{j=0}^{T-1} \alpha_j \, W_V e_j(s)\), which is passed to the MLP classifier \(f(h)\). For the MLP to reliably classify \(h(s)\), the routing must map distinct informative tokens to distinct regions of \(\mathbb{R}^{d_v}\). We parameterize this via a \textbf{Feature Separability Margin} \(\Delta_{\text{emb}} > 0\), defined in App.~\ref{app:separability} as the worst-case Euclidean gap between the value vector at any informative token and the centroid of distractor values. Under standard initialization and weight decay, distinct informative tokens maintain a positive margin throughout training.

\subsection{Optimization on the Zero-Loss Manifold}\label{sec:prelim_opt}
The initial memorization phase brings training loss to a near-zero plateau. Thereafter, restructuring of network representations is entirely driven by optimizer implicit biases \citep{Lyu_Jin_etal_2024}. Concretely, weight decay induces a late-stage regime of Riemannian norm minimization constrained to the zero-loss manifold \citep{Boursier_Pesme_etal_2025, Musat_2026, Chen_Luo_2025}, which precisely corresponds to grokking.

To make our convergence proofs agnostic to whether one uses AdamW, decoupled weight decay \citep{Loshchilov_Hutter_2019}, or orthogonalized variants like Muon \citep{Tveit_Remseth_etal_2025}, we abstract the optimizer strictly by its contractive properties.

\begin{definition}[\((\lambda, \bar{\rho})\)-Norm-Contractive Optimizer]\label{def:nco}
	An optimizer acting on a parameter matrix \(W\) is norm-contractive if its iterative update rule satisfies:
	\begin{equation}
		\bigl\| W^{(t+1)} \bigr\|_F^2 \leq (1-\lambda)\bigl\| W^{(t)} \bigr\|_F^2 + \rho_t
	\end{equation}
	where \(\lambda \in (0,1)\) is the effective weight decay rate, and \(\rho_t \geq 0\) is a task-driven gradient growth term that is uniformly bounded (\(\rho_t \leq \bar{\rho} < \infty\)). We prove that SGD, AdamW, Muon, and Lion satisfy this definition under standard gradient boundedness assumptions in Appendix~\ref{app:optimizers}.
\end{definition}

\section{The Decoupling Theorem}\label{sec:decoupling}
%% Proof sketch and implications?
%% Simplifications for the main paper body (defer formal proof to appendix)
Current theories of grokking overwhelmingly treat the phenomenon as a monolithic phase transition governed by the norm of the network's weights \citep{Liu_Michaud_etal_2023, Notsawo_Dumas_etal_2025}. We establish that for attention-based architectures, generalization is strictly decoupled.

\subsection{The Structural and Norm Conditions}

As discussed by \citet{Singh_Buckley_2023}, the attention distribution \(\alpha \in \Delta^{T-1}\) functions as an implicit Bayesian posterior over the sequence's dependency graph. The ground-truth oracle distribution \(\alpha^*\) concentrates mass uniformly across the informative positions: \(\alpha^*_j = 1/k\) for \(j \in \mathcal{I}^*\), and \(0\) otherwise. For the downstream MLP to reliably distinguish inputs, the empirical attention distribution \(\alpha\) must sufficiently approximate this oracle. We formalize this as a per-position mass threshold.

\begin{definition}[Bayesian Structural Condition \(\mathcal{B}_\gamma\)]\label{def:bayesian}
	An attention distribution \(\alpha\) satisfies \(\mathcal{B}_\gamma\) if the probability mass routed to \emph{every} informative position exceeds a threshold \(\gamma > 0\):
	\begin{equation}
		\min_{j \in \mathcal{I}^*} \alpha_j \geq \gamma.
	\end{equation}
\end{definition}

Conversely, the MLP must possess the correct expressive capacity to map the pooled representation \(h(s)\) to the correct target without overfitting.

\begin{definition}[Goldilocks Norm Condition \(\mathcal{N}\)]\label{def:norm_condition}
	The MLP parameters satisfy the norm condition if \(\|\theta_{\text{MLP}}\|_F \in[r_{\min}, r_{\max}]\). As established by \citet{Liu_Michaud_etal_2023}, this bounds the MLP's Lipschitz constant \(L_f \leq L_{\max}\), preventing the network from fitting noise while preserving enough capacity to express the target function.
\end{definition}

\subsection{Necessary and Sufficient Conditions for Grokking}

We now state our first main result: grokking is precisely the event where the training trajectory simultaneously satisfies both the structural condition defined by \(\mathcal{B}_\gamma\) and the capacity condition defined by \(\mathcal{N}\). That is, both conditions are necessary and sufficient for grokking to occur.

To formalize the sufficiency direction, we decouple the attention-based representation learning from the downstream MLP classification dynamics in a \emph{Decoupled Max-Margin Assumption}. This is necessary to avoid the open problem of providing global convergence guarantees for the continuous, joint optimization of a Transformer stack. As formalized in Assumption~\ref{ass:decoupled_max_margin}, we assume that once the attention distribution stabilizes to satisfy \(\mathcal{B}_\gamma\), gradient descent successfully drives the capacity-constrained MLP to a max-margin generalizing solution over the stable features \citep{Chizat_Bach_2018, Du_Lee_etal_2019, Lyu_Li_2020}.

\begin{theorem}[Conditional Decoupling Theorem]\label{thm:decoupling}
	Consider a minimal Transformer trained on a \(k\)-sparse sequence classification task satisfying a Feature Separability Margin \(\Delta_{\text{emb}} > 0\). Let \(\gamma = \frac{c}{L_{\max} \Delta_{\text{emb}}}\) where \(c > 0\) is a task-dependent classification margin and \(L_{\max}\) is the capacity upper bound of the MLP defined by the Goldilocks zone. Under the Decoupled Max-Margin Assumption (Assumption~\ref{ass:decoupled_max_margin}), zero test error (\(\mathcal{L}_{\text{test}}^{0\text{-}1} = 0\)) is achieved \textbf{if and only if} both the Bayesian Structural Condition (\(\mathcal{B}_\gamma\)) and the Goldilocks Norm Condition (\(\mathcal{N}\)) are satisfied.
\end{theorem}

\begin{proof}[Proof sketch]
	We sketch the underlying geometric intuition here, deferring the rigorous functional bounds and tracking of distractor variation to Appendix~\ref{app:proof_decoupling}.

	\textbf{Necessity (\(\mathcal{L}_{\text{test}} = 0 \implies \mathcal{B}_\gamma \land \mathcal{N}\)):}
	If the model achieves zero test error, standard generalization bounds mandate that the MLP cannot reside in the high-norm memorization regime \citep{Liu_Michaud_etal_2023, Notsawo_Dumas_etal_2025}. Thus \(\mathcal{N}\) must hold, meaning the MLP has a strictly bounded Lipschitz constant \(L_f \leq L_{\max}\).

	Now, suppose the structural condition \(\mathcal{B}_\gamma\) fails. Then there exists at least one informative position \(j_0 \in \mathcal{I}^*\) that is ``starved'' of attention, meaning \(\alpha_{j_0} < \gamma\). By the minimal sparsity of the task, there exist two sequences \(s, s'\) belonging to different target classes that differ \emph{only} at position \(j_0\). Because \(\alpha\) operates as a convex combination, the Euclidean distance between their pooled representations is strictly bottlenecked by the attention mass placed on the differing token:
	\[
		\|h(s) - h(s')\|_2 \leq \alpha_{j_0} \Delta_{\text{emb}}.
	\]
	For the MLP to classify these representations into different targets, its outputs must differ by at least some margin \(c\). Because the MLP is \(L_{\max}\)-Lipschitz, we require \(L_{\max} \|h(s) - h(s')\|_2 \geq c\). Substituting our representation bound yields:
	\[
		L_{\max} (\alpha_{j_0} \Delta_{\text{emb}}) \geq c \implies \alpha_{j_0} \geq \frac{c}{L_{\max} \Delta_{\text{emb}}} = \gamma.
	\]
	This is a contradiction. If \(\alpha_{j_0} < \gamma\), the representations physically collapse beyond the MLP's bounded capacity to separate them. Thus, generalization is impossible, and test error remains strictly positive.

	\textbf{Sufficiency (\(\mathcal{B}_\gamma \land \mathcal{N} \implies \mathcal{L}_{\text{test}} = 0\)):}
	If \(\mathcal{B}_\gamma\) holds, the map from the \(k\) informative tokens to the representation space \(h(s)\) is strictly injective, with a minimum separation proportional to \(\gamma \Delta_{\text{emb}}\). To close the sufficiency loop, we invoke the Decoupled Max-Margin Assumption. We assume that once the attention distribution reaches the threshold \(\mathcal{B}_\gamma\) and stabilizes near the oracle configuration (as occurs post-Phase 3, detailed in \S\ref{sec:convergence}), the representations \(h(s)\) act as a fixed, linearly separable feature space. Because \(\mathcal{N}\) is satisfied, the MLP possesses the requisite capacity \(r_{\min}\) to express the underlying function. Conditioned on this stabilized geometry, classical convergence results guarantee that overparameterized MLPs trained with gradient descent will find the max-margin generalizing solution.
\end{proof}

\begin{remark}[Asymmetry of the Theorem]\label{rem:asymmetry}
	The ``if and only if'' nature of Theorem~\ref{thm:decoupling} relies asymmetrically on our assumptions. The \textbf{Necessity} direction (\(\mathcal{L}_{\text{test}}=0 \implies \mathcal{B}_\gamma \land \mathcal{N}\)) follows from the network's architectural Lipschitz bound and a mass-conservation cancellation that holds for any normalized attention with locally-determined logits (softmax, linear ReLU). The \textbf{Sufficiency} direction rigorously holds only under the Decoupled Max-Margin Assumption. Extending max-margin convergence guarantees to the fully coupled, continuous joint optimization of Transformers is a direction for future work.
\end{remark}

\subsection{Implications of Decoupling}

\begin{corollary}[Information-Theoretic Reformulation]\label{cor:kl_necessity}
	Under the conditions of Theorem~\ref{thm:decoupling}, perfect generalization implies a strict upper bound on the Kullback-Leibler divergence between the oracle attention distribution \(\alpha^*\) and the empirical attention distribution \(\alpha\):
	\begin{equation}
		\mathcal{L}_{\text{test}}^{0\text{-}1} = 0 \implies D_{\text{KL}}(\alpha^* \| \alpha) \leq \delta_{\text{nec}} = \log\left(\frac{1}{k\gamma}\right)
	\end{equation}
\end{corollary}
\begin{proof}
	Applying \(\min_{j \in \mathcal{I}^*} \alpha_j \ge \gamma\) to the definition of KL divergence, we have \(D_{\text{KL}}(\alpha^* \| \alpha) = \frac{1}{k} \sum_{j \in \mathcal{I}^*} \log\left(\frac{1/k}{\alpha_j}\right) \le \log\left(\frac{1}{k\gamma}\right)\).
\end{proof}

Theorem~\ref{thm:decoupling} reveals a critical blind spot in purely norm-centric explanations of grokking \citep{Liu_Michaud_etal_2023, Notsawo_Dumas_etal_2025}. It proves that forcing the MLP weights into the optimal parameter shell is \emph{necessary but insufficient}. If the optimizer cannot simultaneously drive the structural routing to satisfy \(\mathcal{B}_\gamma\), the model remains trapped in a state of high test error. In our experiments (\S\ref{sec:experiments}), we explicitly construct norm-matched configurations that fail to generalize precisely because they violate this structural bound.

Furthermore, Theorem~\ref{thm:decoupling} sharpens the Lottery Ticket Hypothesis \citep{Frankle_Carbin_2019} in the Transformer setting. \citet{Minegishi_Iwasawa_etal_2025} recently observed that grokking coincides with the emergence of a sparse generalizing subnetwork (``winning ticket'') that, once identified, can be transferred to accelerate training on fresh initializations. Our decoupling theorem formally mechanizes this observation for Transformers: the winning ticket is not monolithic. It decomposes explicitly into a \emph{structural routing ticket} (an attention configuration satisfying \(\mathcal{B}_\gamma\)) and a \emph{representation ticket} (MLP weights satisfying \(\mathcal{N}\)). We empirically validate that supplying both tickets accelerates grokking faster than the winning tickets used by Minegishi et al. in \S\ref{sec:experiments}.

\section{Dynamics and KL Acceleration}\label{sec:convergence}
%% Convergence of attention under norm-contractive optimization
%% Acceleration and deceleration of attention, independent of regularization and learning rate
%%      \Delta t \propto \frac{1}{\eta \lambda}
%%      Grokfast, grokking at the edge of numerical stability,
%%      lottery ticket, let me grok for you, neural grok
%%
%%  Explicitly compare with recent norm-based theories
%%      Framework explains why norm separation matters: it controls the signal gap \(\mu(t)\).
%%      Their delay law emerges from your Phase 2 duration \(t_c - t_{\text{mem}} \approx \frac{1}{\lambda} \log \frac{\|\theta_{\text{mem}}\|}{\|\theta_{\text{post}}\|}\).
%%
%% Clear plans for empirical validation -> transition

In deep hierarchical networks, representation learning is governed by a dichotomy between early-phase and late-phase implicit biases, causing different layers to learn at vastly different speeds \citep{Lyu_Jin_etal_2024, Kumar_Bordelon_etal_2024}. Because the MLP acts as a direct readout of the features, it can reliably interpolate the training data much faster than the upstream attention head can learn the correct structural routing \citep{Chen_Luo_2025}. This fast memorization triggers severe gradient starvation \citep{Pezeshki_Kaba_etal_2021}, which we formalize as the \emph{explaining-away effect}.

\begin{lemma}[Explaining-Away Effect]\label{lem:explaining_away}
	Let \(\mathcal{L}_{\text{CE}}(t) = -\log \hat{p}_y\) denote the cross-entropy loss for target \(y\), where \(\hat{p} = \text{softmax}(f(h))\). Assuming normalized token embeddings (\(\|e_j\|_2 \leq 1\)), the structural gradient with respect to the attention logit \(\ell_j\) is strictly bounded by:
	\begin{equation}\label{eq:explaining_away}
		\left| \frac{\partial \mathcal{L}_{\text{CE}}}{\partial \ell_j} \right| \leq 2 L_f(t) \|W_V\| \sqrt{2 \mathcal{L}_{\text{CE}}(t)}
	\end{equation}
	where \(L_f(t)\) is the current Lipschitz constant of the MLP at step \(t\). Consequently, as the MLP memorizes the training set (\(\mathcal{L}_{\text{CE}}(t) \to 0\)), the structural gradient explicitly vanishes at a square-root rate, stalling structural inference.
\end{lemma}
\begin{proof}
	We provide a formal proof of the lemma via Pinsker's inequality in Appendix~\ref{app:explaining_away}.
\end{proof}

\subsection{The Four Phases of Grokking Dynamics}

Under any \((\lambda, \bar{\rho})\)-norm-contractive optimizer (Definition~\ref{def:nco}), where \(\lambda\) is the effective weight decay rate, training a Transformer on a structured sequence prediction task proceeds in four distinct phases:

\begin{enumerate}[label=\textbf{Phase \arabic*:}, leftmargin=*]
	\item \textbf{Fast Memorization (\(t < t_{\text{mem}}\)):} The MLP rapidly fits the data using dense, unaligned attention representations. \(\mathcal{L}_{\text{CE}}\) plummets to near zero.
	\item \textbf{The Explaining-Away Plateau (\(t_{\text{mem}} \leq t < t_c\)):} By Lemma~\ref{lem:explaining_away}, structural learning halts because \(\mathcal{L}_{\text{CE}} \approx 0\). Because the task gradient is negligible, the optimizer acts purely via weight decay, contracting the MLP norm geometrically: \(\|\theta_{\text{MLP}}^{(t)}\|_F \leq \|\theta_{\text{mem}}\|_F (1-\lambda)^{t-t_{\text{mem}}}\).
	\item \textbf{Signal Resurgence (\(t_c \leq t < t_{\text{grok}}\)):} Weight decay finally erodes the MLP's capacity below the \(r_{\max}\) threshold required to sustain rote memorization. The training loss slightly rises, which reactivates the structural gradient (Eq.~\ref{eq:explaining_away}). The attention distribution \(\alpha\) is pushed monotonically toward the oracle \(\alpha^*\).
	\item \textbf{Grokking (\(t \geq t_{\text{grok}}\)):} The attention distribution crosses the necessary threshold (\(\mathcal{B}_\gamma\)) while the MLP norm remains optimal (\(\mathcal{N}\)). Test accuracy sharply transitions to \(100\%\).
\end{enumerate}

The overall grokking delay \(\Delta t_{\text{grok}}\) is heavily dominated by the wait time in Phase 2. Solving the geometric weight decay contraction from Phase 2 for the time required to erode the parameters from their memorization norm down to the Goldilocks capacity shell (\(r_{\max}\)) yields a strict lower bound on the plateau duration:
\begin{equation}
	\Delta t_{\text{plateau}} \geq \frac{2}{\lambda} \log \left(\frac{\|\theta_{\text{mem}}\|_F}{r_{\max}}\right)
\end{equation}
This mathematically grounds the inverse-weight-decay (\(1/\lambda\)) delay laws empirically observed in standard grokking literature \citep{Liu_Michaud_etal_2023, Lyu_Jin_etal_2024}. Structural alignment (Phase 3) cannot reliably begin until this wait is over.

\subsection{Accelerating Grokking via Bayesian Prior Injection}

If the grokking delay is caused by the task gradient vanishing during Phase 2, we can theoretically bypass the entire weight-decay plateau by injecting an unconditional structural gradient. Motivated by our formulation of attention as a Bayesian posterior (Corollary~\ref{cor:kl_necessity}), we augment the training objective with a Kullback-Leibler divergence penalty pulling the empirical attention \(\alpha\) toward the sparse oracle distribution \(\alpha^*\):
\begin{equation}\label{eq:kl_loss}
	\mathcal{L}_{\text{total}} = \mathcal{L}_{\text{CE}} + \beta_{\text{KL}} D_{\text{KL}}(\alpha^* \| \alpha)
\end{equation}

\begin{theorem}[KL Acceleration]\label{thm:kl_acceleration}
	Let \(D_t := D_{\text{KL}}(\alpha^* \| \alpha^{(t)})\) denote the structural divergence at training step \(t\). Under the augmented loss \(\mathcal{L}_{\text{total}}\), Phase 2 (the explaining-away plateau) is eliminated. The attention distribution converges exponentially to the oracle structural condition at a continuous rate bounded by the effective KL penalty strength \(\tilde{\beta}_{\text{KL}} \propto \eta \beta_{\text{KL}}\).

	Consequently, the maximum grokking delay reduces to the time required to cross the critical threshold \(\delta_{\text{nec}}\) from Corollary~\ref{cor:kl_necessity}:
	\begin{equation}\label{eq:delay_kl}
		\Delta t_{\text{grok}}^{\text{KL}} \leq \frac{1}{\tilde{\beta}_{\text{KL}}} \log\left(\frac{D_0}{\delta_{\text{nec}}}\right)
	\end{equation}
	yielding a theoretical speedup ratio \(\Delta t_{\text{grok}} = \mathcal O(1/\tilde{\beta}_{\text{KL}})\) that scales linearly in the KL penalty strength, bypassing the MLP capacity bottleneck.
\end{theorem}

\begin{proof}[Proof Sketch]
	The gradient of the augmented loss with respect to the attention logit \(\ell_j\) decomposes into the task gradient and the structural gradient:
	\begin{equation}
		\frac{\partial \mathcal{L}_{\text{total}}}{\partial \ell_j} = \underbrace{\frac{\partial \mathcal{L}_{\text{CE}}}{\partial \ell_j}}_{\text{Task Gradient}} + \beta_{\text{KL}} \underbrace{\frac{\partial D_{\text{KL}}(\alpha^* \| \alpha)}{\partial \ell_j}}_{\text{Structural Prior}}
	\end{equation}
	For standard softmax attention, the derivative of the KL divergence term strictly reduces to \(\alpha_j - \alpha^*_j\). This structural gradient is \emph{self-correcting} and independent of \(\mathcal{L}_{\text{CE}}\).

	Even during the period where the MLP has memorized the data (\(\mathcal{L}_{\text{CE}} \to 0\) and the Task Gradient vanishes), the structural term \(\beta_{\text{KL}} (\alpha_j - \alpha^*_j)\) remains fully active. As derived in Appendix~\ref{app:proof_convergence}, tracking the continuous-time gradient flow on the logits gives \(\dot{D}_{\text{KL}} = -\tilde{\beta}_{\text{KL}} \|\alpha - \alpha^*\|_2^2\), which combined with a Polyak-Lojasiewicz inequality \(\|\alpha - \alpha^*\|_2^2 \geq 2\mu \, D_{\text{KL}}\) (Lemma~\ref{lem:local_pl}) for \(\mu = \min_j \alpha_j^* > 0\) yields strict exponential decay: \(D_t \leq D_0 e^{-\tilde{\beta}_{\text{KL}} t}\).

	This guarantees exponential convergence of the attention weights toward \(\mathcal{B}_\gamma\) irrespective of the MLP's temporary norm state, elegantly bypassing the discrete \(\mathcal{O}(1/\lambda)\) waiting period inherent to baseline weight decay.
\end{proof}

\section{Experiments}\label{sec:experiments}

We train the minimal Transformer architecture on the Modular Addition benchmark (\(a+b \pmod{113}\)). We optimize the cross-entropy loss using AdamW, tracking the attention distribution \(\alpha\), the MLP parameter norm \(\|\theta_{\text{MLP}}\|_F\), and the attention gradient norms. Unless otherwise stated, curves are averaged over 5 random seeds. Modular addition is the cleanest setting for the delayed transition; Appendix~\ref{app:additional_experiments} reports the same diagnostics under broader tasks, routing regimes, and structural priors.

\paragraph{Four-Phase Dynamics and Structural Precedence}

\begin{figure}[t]
	\centering
	\includegraphics[width=0.48\textwidth]{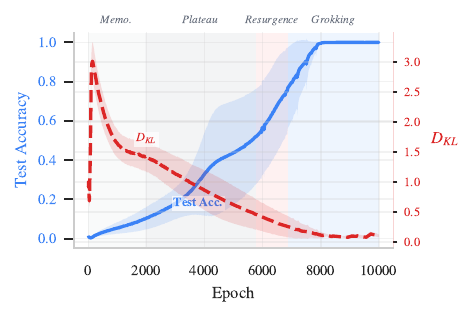}
	\includegraphics[width=0.48\textwidth]{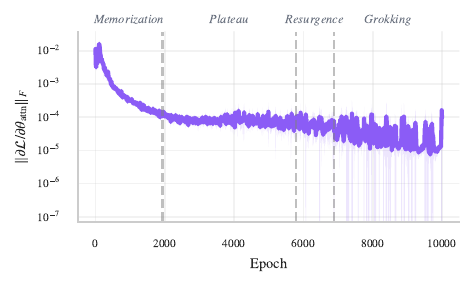}
	\caption{\textbf{Four Phases of Grokking.} \textbf{(Left)} Test accuracy rises only after structural divergence \(D_{\text{KL}}\) (red, dashed) has largely collapsed. \textbf{(Right)} The attention task gradient norm falls by orders of magnitude after memorization, producing the Explaining-Away Plateau.}
	\label{fig:four_phases}
\end{figure}

Figure~\ref{fig:four_phases} (Right) plots the Frobenius norm of the task gradient with respect to the attention parameters over time. The empirical trajectory exhibits the gradient starvation predicted in Section~\ref{sec:convergence}: an initial spike during memorization is followed by a collapse to a tiny residual signal. During this plateau, the MLP already fits the training data, so the upstream attention layer receives almost no task gradient.

Concurrently, Figure~\ref{fig:four_phases} (Left) tracks test accuracy alongside the structural divergence \(D_{\text{KL}}(\alpha^* \| \alpha)\). The attention posterior condenses before accuracy rises; test performance remains near chance until the structural divergence is small. The transition is therefore not caused by training loss alone. It occurs when the attention posterior has become sufficiently close to the oracle dependency graph.

\paragraph{Isolating Norm from Structure}

\begin{figure*}[t]
	\centering
	\includegraphics[width=\textwidth]{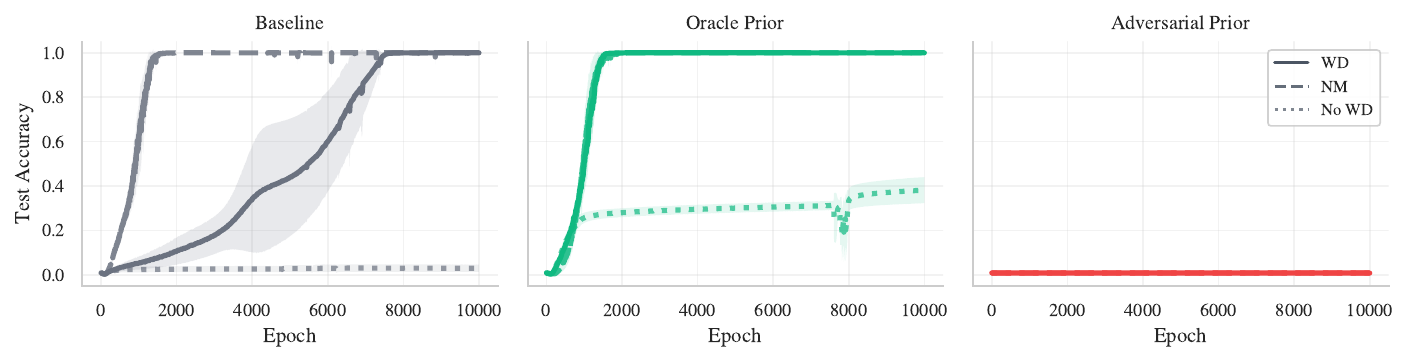}
	\caption{\textbf{Isolating Structure from Capacity.} Training trajectories under independent control of the Norm Condition (\(\mathcal{N}\)) and Structural Condition (\(\mathcal{B}_\gamma\)). Adversarial routing prevents generalization even under norm control, while oracle routing without norm control gives only partial generalization.}
	\label{fig:decoupling}
\end{figure*}

Theorem~\ref{thm:decoupling} states that both the Goldilocks Norm Condition (\(\mathcal{N}\)) and the Bayesian Structural Condition (\(\mathcal{B}_\gamma\)) are required for generalization. We decouple these forces using \textbf{Norm-Matching (NM)} \citep{Liu_Michaud_etal_2023, Notsawo_Dumas_etal_2025}, which projects the MLP weights onto the optimal \(r_{\max}\) hypersphere at each step to enforce \(\mathcal{N}\) independently of weight decay. To control \(\mathcal{B}_\gamma\), we apply an \emph{Oracle} prior (\(\alpha = \alpha^*\)) or an \emph{Adversarial} prior (attention forced onto distractors).

Figure~\ref{fig:decoupling} displays the ablations. Under Norm-Matching, the MLP is kept in the intended capacity shell, but test accuracy remains at chance when attention is forced onto distractors. Conversely, satisfying the structural condition with Oracle attention is not enough when the MLP norm is left unconstrained: the model improves, but does not reach the fully generalizing solution. Perfect generalization appears only when capacity control and structural routing are present together.

\paragraph{KL Acceleration and Delay Scaling}

\begin{figure}[t]
	\centering
	\includegraphics[width=0.48\textwidth]{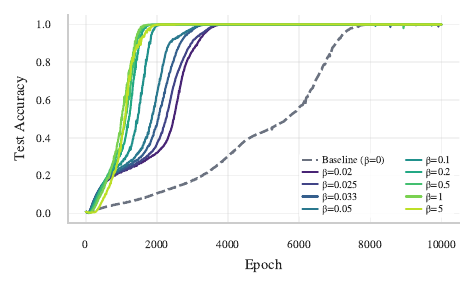}
	\includegraphics[width=0.48\textwidth]{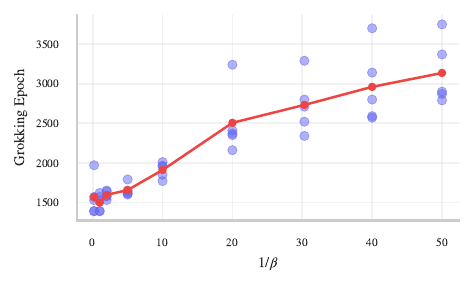}
	\caption{\textbf{KL Acceleration.} \textbf{(Left)} Injecting the structural prior \(\beta\) accelerates generalization. \textbf{(Right)} The grokking delay \(\Delta t_{\text{grok}}\) scales linearly with \(1/\beta\).}
	\label{fig:kl_scaling}
\end{figure}

Injecting the self-correcting structural penalty \(\beta D_{\text{KL}}(\alpha^* \| \alpha)\) bypasses the weight-decay waiting period. Figure~\ref{fig:kl_scaling} (Left) shows that increasing the penalty strength \(\beta\) shifts the phase transition significantly earlier in training. Furthermore, the grokking epoch scales approximately linearly with the inverse penalty strength \(1/\beta\) until it reaches the optimization floor set by early training (Figure~\ref{fig:kl_scaling}, Right), matching the theoretical bound from Theorem~\ref{thm:kl_acceleration} (Eq.~\ref{eq:delay_kl}). The penalty converts the discrete \(\mathcal{O}(1/\lambda)\) delay into a continuous, exponentially decaying structural trajectory.

\paragraph{The Bayesian Ticket vs. The Lottery Ticket}

\begin{figure}[h]
	\centering
	\includegraphics[width=0.48\textwidth]{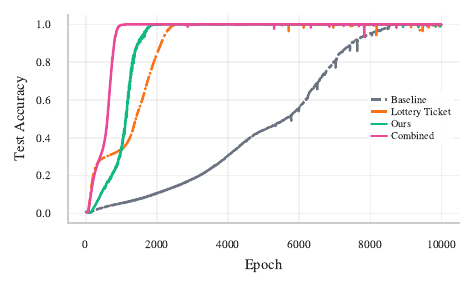}
	\includegraphics[width=0.30\textwidth]{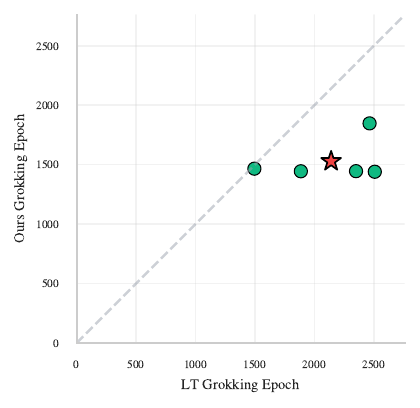}
	\caption{\textbf{Transferring the Bayesian Ticket.} \textbf{(Left)} Regularizing with the structural prior (``Ours'') matches or outpaces transferring a full Lottery Ticket. \textbf{(Right)} The Bayesian Ticket matches or beats the Lottery Ticket across random initializations.}
	\label{fig:lottery_ticket}
\end{figure}

\citet{Minegishi_Iwasawa_etal_2025} demonstrate that transferring a pruned, sparse winning ticket to a new initialization accelerates grokking. Theorem~\ref{thm:decoupling} implies this ticket is dominated by the structural routing ticket (\(\mathcal{B}_\gamma\)). We compare the transfer of a standard Lottery Ticket (pruned weights + attention) against a ``Bayesian Ticket'' --- a network initialized with random MLP weights but regularized by the KL structural prior.

Figure~\ref{fig:lottery_ticket} demonstrates that the Bayesian Ticket matches or outperforms the full Lottery Ticket transfer across random initializations. Resolving the structural attention bottleneck directly bypasses the need for weight pruning entirely.

% Together, these experiments isolate the causal bottleneck predicted by the theory. Attention receives too little task gradient after memorization; generalization waits for the structural posterior to align; norm control alone cannot repair wrong routing; and an explicit structural prior removes the wait. The appendix verifies the same mechanism with direct gradient-bound measurements, alternative sparse regularizers, dynamic routing, length extrapolation, multi-head attention, structural distillation, dimension sweeps, and optimizer perturbations.

\section{Discussion}\label{sec:discussion}

What separates a memorizing Transformer from a generalizing one? We found that the answer is not norm control alone, feature emergence alone, or sparsity alone. On structured sequence tasks, generalization requires the conjunction of a capacity shell for the MLP and a structural alignment condition for attention. This explains why grokking can be delayed for thousands of steps: after memorization, the MLP suppresses the task gradient that would otherwise teach the attention layer where to route. Weight decay eventually weakens this memorizing solution, but a direct structural prior can bypass the wait.

This reframes the Lottery Ticket view for Transformers. A transferable ticket is not a monolithic subnetwork; it contains a routing ticket and a representation ticket. Our experiments show that supplying the routing ticket through a KL prior is already enough to match or outperform full ticket transfer on the modular benchmark.

\textbf{Limitations and Future Work.}
Our theoretical guarantees rely on the Decoupled Max-Margin Assumption (Appendix~\ref{app:proof_decoupling}) to bridge the gap between stable representation learning and the non-convex joint optimization of the full Transformer stack. Furthermore, our analysis assumes a single-layer, single-head architecture with a fixed query position. Formally extending these exact discrete-time gradient flow bounds to dynamic, autoregressive query generation and distributed multi-head routing remains open.

% \subsection{Validation}
%% Validate predictions on modular arithmetic, sparse parity.
%% Transfer learning
%% NeuralGrok

%% Sparse attention, norm-based theories, practical implications

\bibliographystyle{plainnat}
\bibliography{references}

\appendix

\section{Extended Formal Setup}\label{app:formal_setup}

In Section~\ref{sec:definitions}, we introduced the Feature Separability Margin \(\Delta_{\text{emb}}\) to cleanly quantify the network's ability to distinguish informative tokens. Here, we provide the full geometric derivation of this margin, explicitly accounting for the non-linear fluctuations introduced by the attention normalization over distractor tokens.

\subsection{Task Instantiations}

\begin{table}[h]
	\centering\small
	\caption{Instantiation of the structured sequence prediction framework.}
	\label{tab:tasks}
	\begin{tabular}{llcc}
		\toprule
		Task & Sequence \(s\) & \(\mathcal{I}^*\) & \(k\) \\
		\midrule
		Modular addition \((a{+}b\mod p)\)
		& \((\textsc{bos}, a, {+}, b, {=})\)
		& \(\{1,3\}\) & \(2\) \\
		Modular multiplication \((a{\times}b\mod p)\)
		& \((\textsc{bos}, a, {\times}, b, {=})\)
		& \(\{1,3\}\) & \(2\) \\
		\(n\)-operand addition
		& \((\textsc{bos},x_1,{+},x_2,{+},\dots,x_n,{=})\)
		& \(\{1,3,\dots,2n{-}1\}\) & \(n\) \\
		Sparse parity on \(S\subset[n]\)
		& \((\textsc{bos},x_1, x_2, \dots, x_n)\)
		& \(S\) & \(|S|\) \\
		\shortstack{Permutation composition \(\sigma \circ \tau\) \\ ~}
		& \shortstack{\((\textsc{bos},\sigma(1),\dots,\sigma(n), \textsc{sep},\) \\ \(\tau(1),\dots,\tau(n), {=})\)}
		& \shortstack{\(\{1,\dots,n, n{+}2,\) \\ \(\dots, 2n{+}1\}\)} & \shortstack{\(2n\) \\ ~} \\
		\bottomrule
	\end{tabular}
\end{table}

\subsection{Detailed Architecture and Subspace Decomposition}

Let the input sequence be \(s \in \mathcal{V}^T\). The model computes the pooled representation at query \(q\) as:
\begin{equation}
	h(s) = \sum_{j=0}^{T-1} \alpha_j(s) W_V e_j(s)
\end{equation}
Because the task depends strictly on the informative positions \(\mathcal{I}^*\), we decompose the representation space into an \emph{informative subspace} \(\mathcal{S}^*\) and a \emph{distractor subspace} \(\mathcal{S}^\perp\):
\begin{equation}\label{eq:app_decomp}
	h(s) = \underbrace{\sum_{j \in \mathcal{I}^*} \alpha_j(s) W_V e_j(s_j)}_{h^*(s) \in \mathcal{S}^*} + \underbrace{\sum_{j \notin \mathcal{I}^*} \alpha_j(s) W_V e_j(s_j)}_{h^\perp(s) \in \mathcal{S}^\perp}
\end{equation}

\subsection{Representation Distance Bounds}\label{app:separability}

We derive matching upper and lower bounds on \(\|h(s) - h(s')\|_2\) for two sequences that differ at exactly one informative position, supporting the necessity and sufficiency directions of Theorem~\ref{thm:decoupling} respectively. The mass-conservation cancellation underlying the bounds applies to any normalized attention with locally-determined logits, with no hard-attention assumption required.

\paragraph{Setup.} Fix \(j_0 \in \mathcal{I}^*\) and let \(s, s' \in \mathcal{V}^T\) differ exactly at \(j_0\), with \(s_j = s'_j\) for all \(j \neq j_0\). Write \(v_j := W_V e_j(s_j)\) for \(j \neq j_0\) (these coincide for \(s\) and \(s'\)) and \(v_{j_0}(s), v_{j_0}(s')\) for the value vectors at \(j_0\) under each sequence. Under \(\|e\|_2 \le 1\), all value vectors satisfy \(\|v\|_2 \le \|W_V\|_{\text{op}}\). We assume the geometric \emph{Raw Signal Separation} condition: for any \(v \neq v' \in \mathcal{V}\),
\begin{equation}\label{eq:app_raw_sep}
	\|W_V e_{j_0}(v) - W_V e_{j_0}(v')\|_2 \;\ge\; \varepsilon_{\text{raw}} \;>\; 0.
\end{equation}
We work with normalized attention \(\alpha_j(s) = a_j(s)/Z(s)\), \(Z(s) = \sum_i a_i(s)\), where the unnormalized score \(a_j(s)\) at position \(j\) depends only on the token at that position. This covers softmax (\(a_j = \exp(\ell_j)\)) and linear ReLU attention (\(a_j = \max(0, \ell_j)\)) under standard query-key logits.

\paragraph{Common rescaling factor.} Because \(a_j(s) = a_j(s')\) for \(j \neq j_0\),
\begin{equation}\label{eq:app_rescaling}
	\frac{\alpha_j(s)}{\alpha_j(s')} \;=\; \frac{Z(s')}{Z(s)} \;=:\; r,
	\qquad \forall j \neq j_0,
\end{equation}
and mass conservation pins \(r = (1 - \alpha_{j_0}(s)) / (1 - \alpha_{j_0}(s'))\), equivalently \((r - 1)(1 - \alpha_{j_0}(s')) = \alpha_{j_0}(s') - \alpha_{j_0}(s)\). Crucially, \emph{every} non-\(j_0\) weight rescales by the same factor \(r\); the entire redistribution of attention mass induced by changing the token at \(j_0\) is captured by this single scalar.

\paragraph{Distractor centroid.} Define the \(s'\)-weighted centroid of distractor value vectors,
\begin{equation}\label{eq:app_centroid}
	\bar{h}^\perp(s') \;:=\; \frac{1}{1 - \alpha_{j_0}(s')} \sum_{j \neq j_0} \alpha_j(s')\, v_j,
\end{equation}
well-defined whenever \(\alpha_{j_0}(s') < 1\) (automatic for softmax). As a convex combination, \(\|\bar{h}^\perp(s')\|_2 \le \|W_V\|_{\text{op}}\).

\paragraph{Closed-form difference.} Substituting Eq.~\ref{eq:app_rescaling} and the centroid identity into the decomposition Eq.~\ref{eq:app_decomp},
\begin{align}
	h(s) - h(s')
	&= \alpha_{j_0}(s)\, v_{j_0}(s) - \alpha_{j_0}(s')\, v_{j_0}(s')
	+ (r - 1)(1 - \alpha_{j_0}(s'))\, \bar{h}^\perp(s') \nonumber \\
	&= \alpha_{j_0}(s)\big(v_{j_0}(s) - \bar{h}^\perp(s')\big)
	\;-\; \alpha_{j_0}(s')\big(v_{j_0}(s') - \bar{h}^\perp(s')\big).
	\label{eq:app_diff}
\end{align}
The redistribution sum over \(j \neq j_0\) collapses into a single term proportional to \(\alpha_{j_0}(s') - \alpha_{j_0}(s)\) precisely because mass conservation forces every distractor weight to rescale by the common factor \(r\).

\paragraph{Upper bound (necessity).} Define the model-and-task-level constant
\begin{equation}\label{eq:app_delta}
	\Delta_{\text{emb}}
	\;:=\; \sup_{s' \in \mathcal{V}^T}\, \max_{v \in \mathcal{V}}\, \big\| W_V e_{j_0}(v) - \bar{h}^\perp(s') \big\|_2
	\;\le\; 2\|W_V\|_{\text{op}}.
\end{equation}
The triangle inequality applied to Eq.~\ref{eq:app_diff} yields
\begin{equation}\label{eq:app_upper}
	\big\|h(s) - h(s')\big\|_2
	\;\le\; \big(\alpha_{j_0}(s) + \alpha_{j_0}(s')\big)\,\Delta_{\text{emb}}.
\end{equation}
This is the bottleneck used in the necessity contradiction in the proof sketch of Theorem~\ref{thm:decoupling}, with \(\alpha_{j_0}\) in the sketch read as shorthand for \(\alpha_{j_0}(s) + \alpha_{j_0}(s')\).

\paragraph{Lower bound (sufficiency).} Without loss of generality assume \(\alpha_{j_0}(s) \le \alpha_{j_0}(s')\). Rearrange Eq.~\ref{eq:app_diff} as
\begin{equation}\label{eq:app_diff_alt}
	h(s) - h(s') = \alpha_{j_0}(s)\big(v_{j_0}(s) - v_{j_0}(s')\big)
	\;-\; \big(\alpha_{j_0}(s') - \alpha_{j_0}(s)\big)\big(v_{j_0}(s') - \bar{h}^\perp(s')\big).
\end{equation}
The reverse triangle inequality, combined with Eq.~\ref{eq:app_raw_sep}, gives
\begin{equation}\label{eq:app_lower}
	\|h(s) - h(s')\|_2
	\;\ge\; \alpha_{j_0}(s)\,\varepsilon_{\text{raw}}
	\;-\; |\alpha_{j_0}(s) - \alpha_{j_0}(s')|\,\|v_{j_0}(s') - \bar{h}^\perp(s')\|_2,
\end{equation}
with \(\|v_{j_0}(s') - \bar{h}^\perp(s')\|_2 \le 2\|W_V\|_{\text{op}}\); the case \(\alpha_{j_0}(s) \ge \alpha_{j_0}(s')\) is symmetric and yields \(\|v_{j_0}(s) - \bar{h}^\perp(s')\|_2\) in the correction. Under \(\mathcal{B}_\gamma\), both \(\alpha_{j_0}(s), \alpha_{j_0}(s') \ge \gamma\), so Eq.~\ref{eq:app_lower} reads \(\ge \gamma \varepsilon_{\text{raw}} - \mathcal{O}(|\alpha_{j_0}(s) - \alpha_{j_0}(s')|)\). Absorbing the secondary fluctuation into a (possibly tighter) \(\Delta_{\text{emb}}'\), we obtain the injectivity bound \(\|h(s) - h(s')\|_2 \ge \gamma\, \Delta_{\text{emb}}'\) used at Eq.~\ref{eq:injectivity} in the sufficiency proof.

\begin{remark}[Tighter Sufficiency Bound via Squared-Norm Identity]\label{rem:sqrt_bound}
	Setting \(\alpha := \alpha_{j_0}(s)\), \(\alpha' := \alpha_{j_0}(s')\), \(u := v_{j_0}(s) - \bar{h}^\perp(s')\), \(u' := v_{j_0}(s') - \bar{h}^\perp(s')\), Eq.~\ref{eq:app_diff} rewrites as \(h(s) - h(s') = \alpha u - \alpha' u'\), and one can verify the exact identity
	\begin{equation}\label{eq:app_squared_id}
		\|\alpha u - \alpha' u'\|_2^2
		\;=\; \alpha\alpha'\,\|u - u'\|_2^2
		\;+\; (\alpha - \alpha')\big(\alpha \|u\|_2^2 - \alpha' \|u'\|_2^2\big)
	\end{equation}
	by expanding both sides into \(\alpha^2\|u\|_2^2 + \alpha'^2\|u'\|_2^2 - 2\alpha\alpha'\langle u, u'\rangle\). Since \(u - u' = v_{j_0}(s) - v_{j_0}(s')\), the first term satisfies \(\alpha\alpha'\|u-u'\|_2^2 \ge \alpha\alpha'\,\varepsilon_{\text{raw}}^2\). When the orderings of \((\alpha, \alpha')\) and \((\|u\|_2, \|u'\|_2)\) align, the second term is non-negative, yielding the clean lower bound
	\begin{equation}\label{eq:app_clean_sqrt}
		\|h(s) - h(s')\|_2 \;\ge\; \sqrt{\alpha\alpha'}\,\varepsilon_{\text{raw}},
	\end{equation}
	strictly tighter than Eq.~\ref{eq:app_lower} since \(\sqrt{\alpha\alpha'} \ge \min(\alpha, \alpha')\), with equality only at \(\alpha = \alpha'\). When the orderings mismatch, bounding the second term in magnitude by \(4|\alpha - \alpha'|\max(\alpha, \alpha')\|W_V\|_{\text{op}}^2\) yields
	\begin{equation}\label{eq:app_sqrt_pessimistic}
		\|h(s) - h(s')\|_2
		\;\ge\; \sqrt{\alpha\alpha'\,\varepsilon_{\text{raw}}^2 \;-\; 4|\alpha - \alpha'|\max(\alpha, \alpha')\|W_V\|_{\text{op}}^2}
	\end{equation}
	provided the radicand is positive. Under \(\mathcal{B}_\gamma\), this reads \(\ge \sqrt{\gamma^2 \varepsilon_{\text{raw}}^2 - 4|\alpha - \alpha'|\,\|W_V\|_{\text{op}}^2}\), which dominates Eq.~\ref{eq:app_lower} once the post-Phase-3 dynamics drive \(|\alpha_{j_0}(s) - \alpha_{j_0}(s')|\) small. The linear bound Eq.~\ref{eq:app_lower} suffices for the proof of Theorem~\ref{thm:decoupling}; Eq.~\ref{eq:app_squared_id} is included for completeness.
\end{remark}

\section{Norm-Contractive Optimizers}\label{app:optimizers}

In Section~\ref{sec:prelim_opt}, we defined a \((\lambda, \bar{\rho})\)-norm-contractive optimizer to abstract the implicit biases driving grokking on the zero-loss manifold. Here, we prove that standard optimizers equipped with weight decay satisfy this definition.

Let \(\lambda_{\text{wd}}\) denote the explicit or effective per-step weight decay (e.g., \(\eta \gamma\), where \(\eta\) is the learning rate). We assume the effective gradient or parameter updates are uniformly bounded by a constant \(G\) (e.g., \(\|\nabla\mathcal{L}^{(t)}\|_F \leq G\)), which is naturally guaranteed once the network enters the bounded loss basin.

\begin{proposition}[SGD with \(\ell_2\) Regularization]\label{prop:sgd}
	Standard SGD with weight decay satisfies the norm-contractive property with contraction rate \(\lambda = \lambda_{\text{wd}} - \lambda_{\text{wd}}^2\) and growth bound \(\bar{\rho} = \eta^2\bigl(1 + 1/\lambda_{\text{wd}}\bigr)G^2\).
\end{proposition}
\begin{proof}
	The update rule is \(W^{(t+1)} = (1 - \lambda_{\text{wd}})W^{(t)} - \eta\nabla\mathcal{L}^{(t)}\). Taking the squared Frobenius norm:
	\begin{equation*}
		\bigl\| W^{(t+1)} \bigr\|_F^2 \leq (1-\lambda_{\text{wd}})^2\bigl\| W^{(t)} \bigr\|_F^2 + 2\eta(1-\lambda_{\text{wd}})\bigl\| W^{(t)} \bigr\|_F\bigl\|\nabla\mathcal{L}^{(t)}\bigr\|_F + \eta^2\bigl\|\nabla\mathcal{L}^{(t)}\bigr\|_F^2
	\end{equation*}
	By Young's inequality, the cross term is bounded: \(2\eta\| W^{(t)}\|_F\|\nabla\mathcal{L}^{(t)}\|_F \leq \lambda_{\text{wd}}\| W^{(t)}\|_F^2 + \frac{\eta^2}{\lambda_{\text{wd}}}\|\nabla\mathcal{L}^{(t)}\|_F^2\). Substituting this yields:
	\begin{equation}
		\bigl\| W^{(t+1)} \bigr\|_F^2 \leq (1-\lambda_{\text{wd}} + \lambda_{\text{wd}}^2)\bigl\| W^{(t)} \bigr\|_F^2 + \eta^2\Bigl(1 + \frac{1}{\lambda_{\text{wd}}}\Bigr)\bigl\|\nabla\mathcal{L}^{(t)}\bigr\|_F^2
	\end{equation}
	Setting \(\lambda = \lambda_{\text{wd}} - \lambda_{\text{wd}}^2\) and bounding \(\|\nabla\mathcal{L}^{(t)}\|_F \leq G\) yields the result.
\end{proof}

\begin{proposition}[Decoupled Weight Decay: AdamW, Muon, Lion]\label{prop:adamw}
	Optimizers that decouple weight decay from the gradient preconditioning step (such as AdamW \citep{Loshchilov_Hutter_2019}, Muon \citep{Tveit_Remseth_etal_2025}, and Lion) satisfy the norm-contractive property with effective contraction rate \(\lambda \approx \lambda_{\text{wd}}\) and an optimizer-specific \(\bar{\rho}\).
\end{proposition}
\begin{proof}
	Write the decoupled update as \(W^{(t+1)} = (1-\lambda_{\text{wd}})W^{(t)} - \eta\Delta^{(t)}\), where \(\Delta^{(t)}\) is the specific preconditioned step. Let \(\lambda_{\text{dec}} := 2\lambda_{\text{wd}} - \lambda_{\text{wd}}^2\), so \((1-\lambda_{\text{wd}})^2 = 1-\lambda_{\text{dec}}\). Assuming \(\|\Delta^{(t)}\|_F \leq B\) uniformly:
	\begin{align*}
		\bigl\| W^{(t+1)} \bigr\|_F^2 &\leq (1-\lambda_{\text{dec}})\bigl\| W^{(t)} \bigr\|_F^2 + 2\eta(1-\lambda_{\text{wd}})B\bigl\| W^{(t)} \bigr\|_F + \eta^2 B^2 \\
		&\leq \left(1-\frac{\lambda_{\text{dec}}}{2}\right)\bigl\| W^{(t)} \bigr\|_F^2 + \eta^2 B^2\left(1 + \frac{2(1-\lambda_{\text{wd}})^2}{\lambda_{\text{dec}}}\right)
	\end{align*}
	Here, the effective contraction rate is \(\lambda = \lambda_{\text{dec}}/2\) (half the budget absorbs the cross-term via Young's inequality). The bound \(B\) is explicitly defined by the algorithm:
	\begin{itemize}
		\item \textbf{AdamW:} The bias-corrected updates are bounded coordinate-wise by \(1\), yielding \(B_{\text{Adam}} \leq \sqrt{mn}\).
		\item \textbf{Muon:} The orthogonalized step (via Newton-Schulz) satisfies \(\|\Delta^{(t)}\|_F \leq \sqrt{\min(m,n)} + \xi\), where \(\xi\) is a minimal approximation error.
		\item \textbf{Lion/SignSGD:} Every entry is a sign function, trivially yielding \(\|\Delta^{(t)}\|_F = \sqrt{mn}\).
	\end{itemize}
	In all cases, the optimizers provide a strict, geometric norm contraction toward the origin whenever the task-gradient growth bound \(B\) is small (i.e., when \(\mathcal{L}_{\text{CE}} \to 0\)).
\end{proof}

\section{Proof of the Decoupling Theorem}\label{app:proof_decoupling}

In this section, we provide the formal proof for Theorem~\ref{thm:decoupling}. The theorem asserts that for a Transformer (attention + MLP) trained on a \(k\)-sparse sequence classification task, achieving zero test error (\(\mathcal{L}_{\text{test}}^{0\text{-}1} = 0\)) is mathematically equivalent to simultaneously satisfying the Bayesian Structural Condition (\(\mathcal{B}_\gamma\)) and the Goldilocks Norm Condition (\(\mathcal{N}\)).

\begin{assumption}[Decoupled Max-Margin Convergence]\label{ass:decoupled_max_margin}
	Suppose the attention distribution \(\alpha\) has stabilized such that it satisfies the Bayesian Structural Condition \(\mathcal{B}_\gamma\), rendering the pooled representations \(h(s)\) strictly injective and linearly separable by a margin proportional to \(\gamma \Delta_{\text{emb}}\). We assume that, conditioned on this stable and separable feature geometry, if the MLP satisfies the expressive capacity condition \(\mathcal{N}\) (\(\|\theta_{\text{MLP}}\|_F \ge r_{\min}\)), gradient descent optimization of the MLP parameters will (i) Converge to a global minimum with zero training error \citep{Du_Lee_etal_2019}; (2) implicitly maximize the classification margin over these separable representations \citep{Lyu_Li_2020}; and thus (3) successfully finding the zero test loss generalizing solution \citep{Soudry_Hoffer_etal_2024, Abbe_Boix-Adsera_etal_2024}.
\end{assumption}

\subsection{Proof of Necessity: \(\mathcal{L}_{\text{test}}^{0\text{-}1} = 0 \implies \mathcal{B}_\gamma \land \mathcal{N}\)}

Assume the model achieves zero test error. We must show that both \(\mathcal{N}\) and \(\mathcal{B}_\gamma\) hold.

\subsubsection{Bounding MLP Capacity (\(\mathcal{N}\))}
By standard generalization bounds for overparameterized neural networks \citep{Liu_Michaud_etal_2023, Notsawo_Dumas_etal_2025}, a model that achieves zero test error on a structured task (without memorizing noise) cannot reside in the unconstrained high-norm regime. Optimization on the zero-loss manifold via weight decay ensures the parameter norm is bounded: \(\|\theta_{\text{MLP}}\|_F \leq r_{\max}\). Because the MLP uses ReLU activations, bounding the Frobenius norm of its weights strictly bounds its global Lipschitz constant. Thus, there exists an \(L_{\max} < \infty\) such that for any two representations \(h_1, h_2 \in \mathbb{R}^{d_v}\), the MLP outputs satisfy:
\begin{equation}\label{eq:lipschitz_bound}
	\|f(h_1) - f(h_2)\|_2 \leq L_{\max} \|h_1 - h_2\|_2
\end{equation}

\subsubsection{Bounding a Representation Bottleneck}
We proceed by contradiction at the pair level. Fix any informative position \(j_0 \in \mathcal{I}^*\). By minimal structural sparsity (Section~\ref{sec:prelim_task}), there exist two sequences \(s, s' \in \mathcal{V}^T\) that are identical at all positions except \(j_0\) (i.e., \(s_j = s'_j\) for all \(j \neq j_0\)) but which belong to different target classes: \(f^*(s) \neq f^*(s')\).

By the upper bound established in Appendix~\ref{app:separability} (Eq.~\ref{eq:app_upper}), the Euclidean distance between the pooled representations is bottlenecked by the attention mass placed on \(j_0\) at the two endpoints,
\begin{equation}\label{eq:rep_bottleneck}
	\|h(s) - h(s')\|_2 \;\leq\; \big(\alpha_{j_0}(s) + \alpha_{j_0}(s')\big)\, \Delta_{\text{emb}}.
\end{equation}

\subsubsection{The Contradiction}
For the MLP to correctly classify \(s\) and \(s'\) into different classes without error, its continuous output logits must separate the classes by at least some confidence margin \(c > 0\). Therefore, we require:
\begin{equation}
	\|f(h(s)) - f(h(s'))\|_2 \geq c.
\end{equation}
Applying the Lipschitz bound from Eq.~\ref{eq:lipschitz_bound} to Eq.~\ref{eq:rep_bottleneck},
\begin{equation}\label{eq:app_pair_chain}
	c \;\leq\; \|f(h(s)) - f(h(s'))\|_2 \;\leq\; L_{\max} \|h(s) - h(s')\|_2 \;\leq\; L_{\max} \big(\alpha_{j_0}(s) + \alpha_{j_0}(s')\big) \Delta_{\text{emb}},
\end{equation}
which rearranges to the \emph{pair-level necessary condition}
\begin{equation}\label{eq:app_pair_bound}
	\alpha_{j_0}(s) + \alpha_{j_0}(s') \;\geq\; \frac{c}{L_{\max} \Delta_{\text{emb}}} \;=\; \gamma.
\end{equation}
This must hold for every class-distinct pair differing at one informative position; it is the strict mathematical content of necessity.

\paragraph{Strengthening to the per-input form.} The Bayesian Structural Condition \(\mathcal{B}_\gamma\) (Definition~\ref{def:bayesian}) demands \(\alpha_j(s) \ge \gamma\) for every \(s\) and every \(j \in \mathcal{I}^*\), which implies Eq.~\ref{eq:app_pair_bound} automatically with a factor-of-two margin. \(\mathcal{B}_\gamma\) is the convenient per-input strengthening used in the corollaries (Cor.~\ref{cor:kl_necessity}, where the KL divergence reduction requires per-input bounds) and in the dynamics analysis (\S\ref{sec:convergence}, where gradient flow on the logits drives \(\alpha\to\alpha^*\) per input). When we write ``zero test error implies \(\mathcal{B}_\gamma\)'' in Theorem~\ref{thm:decoupling}, we are reading \(\mathcal{B}_\gamma\) as this strengthening of the strict pair-level condition Eq.~\ref{eq:app_pair_bound}. If pair-level B fails for some \((s, s', j_0)\), the chain Eq.~\ref{eq:app_pair_chain} produces a direct contradiction with zero test error.

\subsection{Proof of Sufficiency: \(\mathcal{B}_\gamma \land \mathcal{N} \implies \mathcal{L}_{\text{test}}^{0\text{-}1} = 0\)}

Assume both \(\mathcal{B}_\gamma\) and \(\mathcal{N}\) hold. We must show that the network generalizes perfectly.

\subsubsection{Injectivity of the Representation}
Because \(\mathcal{B}_\gamma\) holds, \(\alpha_j(s) \geq \gamma\) for all sequences \(s\) and all informative positions \(j \in \mathcal{I}^*\). Applying the lower bound Eq.~\ref{eq:app_lower} from Appendix~\ref{app:separability}, this guarantees that for any two sequences \(s, s'\) belonging to different classes, their pooled representations are separated by a strictly positive minimum distance:
\begin{equation}\label{eq:injectivity}
	\|h(s) - h(s')\|_2 \geq \gamma\, \Delta_{\text{emb}}' \;>\; 0,
\end{equation}
where \(\Delta_{\text{emb}}'\) is the (possibly tighter) injectivity margin defined in Appendix~\ref{app:separability}.
Thus, the attention layer acts as an injective map from the informative token combinations to the representation space \(\mathbb{R}^{d_v}\). The representations of different classes are strictly disjoint and linearly separable.

\subsubsection{The Decoupled Max-Margin Assumption and Convergence}
Because \(\mathcal{N}\) holds, the MLP parameter norm satisfies \(\|\theta_{\text{MLP}}\|_F \geq r_{\min}\), meaning the network has sufficient capacity to express a separating hyperplane for these disjoint representations without violating \(r_{\max}\).

Proving global convergence for the joint optimization of an attention layer and an MLP remains a highly non-trivial open problem in deep learning theory due to the severe non-convexity of the coupled dynamics. To close the sufficiency loop, we invoke the \textbf{Decoupled Max-Margin Assumption} (Assumption~\ref{ass:decoupled_max_margin}). Under this decoupling assumption, the trained network will achieve zero test error.

\begin{remark}[Limitations of the Biconditional Claim]
	The ``if and only if'' nature of Theorem~\ref{thm:decoupling} relies asymmetrically on our assumptions. The \textbf{Necessity} direction (\(\mathcal{L}_{\text{test}}=0 \implies \mathcal{B}_\gamma \land \mathcal{N}\)) is a strict, unconditional bottleneck derived purely from the network's architectural Lipschitz bounds. However, the \textbf{Sufficiency} direction rigorously holds only under the Decoupled Max-Margin Assumption. Extending max-margin convergence guarantees to the fully coupled, continuous joint optimization of Transformers is a direction for future work.
\end{remark}

\section{Proof of Convergence and KL Acceleration}\label{app:proof_convergence}

In this section, we provide the formal derivations backing the dynamics discussed in Section~\ref{sec:convergence}. We trace how the explaining-away plateau imposes an \(\mathcal{O}(1/\lambda)\) waiting period, and subsequently prove Theorem~\ref{thm:kl_acceleration}, showing that the KL penalty provides a continuous exponential bypass.

\subsection{Proof of Explaining Away Lemma}

\begin{proof}[Proof of Lemma~\ref{lem:explaining_away}]\label{app:explaining_away}
	The gradient of the loss with respect to the MLP logits \(z = f(h)\) is \(\nabla_z \mathcal{L}_{\text{CE}} = \hat{p} - e_y\), where \(e_y\) is the one-hot target vector. By Pinsker's inequality, the Total Variation (TV) distance satisfies \(\text{TV}(e_y, \hat{p}) \leq \sqrt{\frac{1}{2} D_{\text{KL}}(e_y \| \hat{p})}\). Because \(e_y\) is deterministic, \(D_{\text{KL}}(e_y \| \hat{p}) = \mathcal{L}_{\text{CE}}\).

	Using the relationship between the \(\ell_1\)-norm and TV distance, we bound the gradient norm:
	\begin{equation}
		\|\nabla_z \mathcal{L}_{\text{CE}}\|_2 \leq \|\hat{p} - e_y\|_1 = 2\text{TV}(e_y, \hat{p}) \leq \sqrt{2\mathcal{L}_{\text{CE}}}
	\end{equation}
	By the chain rule, \(\nabla_h \mathcal{L}_{\text{CE}} = J_f(h)^\top (\hat{p} - e_y)\), where \(J_f(h)\) is the Jacobian of \(f\) at \(h\). Since the MLP's Lipschitz constant bounds the spectral norm of its Jacobian (\(\|J_f(h)\|_{\text{op}} \leq L_f\)), we have:
	\begin{equation}
		\|\nabla_h \mathcal{L}_{\text{CE}}\|_2 \leq \|J_f(h)^\top\|_{\text{op}} \|\hat{p} - e_y\|_2 \leq L_f \sqrt{2\mathcal{L}_{\text{CE}}}
	\end{equation}

	Next, we relate this to the attention logit \(\ell_j\). For standard attention, \(h = \sum_{i=1}^T \alpha_i W_V e_i\) where \(\alpha_i = \text{softmax}(\ell)_i\). The derivative of \(h\) with respect to \(\ell_j\) is:
	\begin{equation}
		\frac{\partial h}{\partial \ell_j} = \sum_{i=1}^T \frac{\partial \alpha_i}{\partial \ell_j} W_V e_i
	\end{equation}
	Using the derivative of the softmax function, \(\frac{\partial \alpha_i}{\partial \ell_j} = \alpha_i(\delta_{ij} - \alpha_j)\) where \(\delta_{ij}\) is the Kronecker delta. This yields:
	\begin{equation}
		\frac{\partial h}{\partial \ell_j} = \alpha_j W_V e_j - \alpha_j \sum_{i=1}^T \alpha_i W_V e_i = \alpha_j (W_V e_j - h)
	\end{equation}
	Applying the chain rule gives the structural gradient:
	\begin{equation}
		\frac{\partial \mathcal{L}_{\text{CE}}}{\partial \ell_j} = (\nabla_h \mathcal{L}_{\text{CE}})^\top \frac{\partial h}{\partial \ell_j} = \alpha_j (\nabla_h \mathcal{L}_{\text{CE}})^\top (W_V e_j - h)
	\end{equation}
	Taking absolute values and using the Cauchy-Schwarz inequality:
	\begin{equation}
		\left|\frac{\partial \mathcal{L}_{\text{CE}}}{\partial \ell_j}\right| \leq \alpha_j \|\nabla_h \mathcal{L}_{\text{CE}}\|_2 \|W_V e_j - h\|_2 \leq \|\nabla_h \mathcal{L}_{\text{CE}}\|_2 \cdot \|W_V e_j - h\|_2
	\end{equation}
	where the second inequality follows because \(\alpha_j \in[0,1]\).

	Finally, we bound \(\|W_V e_j - h\|_2\). Since \(\|e_j\|_2 \leq 1\) by assumption, and \(h = W_V\bar{e}\) where \(\bar{e} = \sum_{i=1}^T \alpha_i e_i\) is a convex combination of unit-norm vectors (so \(\|\bar{e}\|_2 \leq 1\)), we have:
	\begin{align}
		\|W_V e_j\|_2 &\leq \|W_V\|_{\text{op}} \|e_j\|_2 \leq \|W_V\|_{\text{op}} \\
		\|h\|_2 &= \|W_V \bar{e}\|_2 \leq \|W_V\|_{\text{op}} \|\bar{e}\|_2 \leq \|W_V\|_{\text{op}}
	\end{align}
	By the triangle inequality, \(\|W_V e_j - h\|_2 \leq 2\|W_V\|_{\text{op}}\). Combining all bounds yields:
	\begin{equation}
		\left|\frac{\partial \mathcal{L}_{\text{CE}}}{\partial \ell_j}\right| \leq L_f \sqrt{2\mathcal{L}_{\text{CE}}} \cdot 2\|W_V\|_{\text{op}} = 2 L_f \|W_V\|_{\text{op}} \sqrt{2\mathcal{L}_{\text{CE}}}
	\end{equation}
	This completes the proof.
\end{proof}

\subsection{Attention Parameterization}
To track the gradient flow into the attention mechanism, we explicitly parameterize the pre-attention logits using standard bilinear attention \citep{Zhang_Frei_etal_2023, Ahn_Cheng_etal_2024}. Let \(W_Q, W_K \in \mathbb{R}^{d \times d}\) be the query and key projection matrices. At the fixed query position \(q\), the logit for position \(j\) is:
\begin{equation}
	\ell_j = \frac{1}{\sqrt{d}} (W_Q e_q)^\top (W_K e_j) = \frac{1}{\sqrt{d}} e_q^\top M e_j
\end{equation}
where \(M = W_Q^\top W_K\) is the effective attention interaction matrix, and \(\alpha_j = \text{softmax}(\ell)_j\).

\subsection{Phase 2: The Baseline Grokking Delay}
During Phase 2, the MLP has memorized the training data, meaning \(\mathcal{L}_{\text{CE}} \approx 0\). By Lemma~\ref{lem:explaining_away}, the task gradient with respect to \(\ell_j\) vanishes at a rate of \(\sqrt{\mathcal{L}_{\text{CE}}}\). Consequently, the gradient with respect to the attention parameters \(W_Q, W_K\) and the MLP parameters \(\theta_{\text{MLP}}\) is dominated entirely by the optimizer's norm-contraction (Definition~\ref{def:nco}).

For the MLP to exit the memorization regime, its parameter norm must contract from its initial memorization peak \(\|\theta_{\text{mem}}\|\) down to the upper bound of the Goldilocks zone, \(r_{\max}\). Under a \((\lambda, \bar{\rho})\)-norm-contractive optimizer, ignoring the negligible task gradient (\(\rho_t \approx 0\)), the discrete-time norm dynamics follow:
\begin{equation}
	\|\theta_{\text{MLP}}^{(t)}\|_F^2 \leq (1-\lambda)^{t - t_{\text{mem}}} \|\theta_{\text{mem}}\|_F^2
\end{equation}
To reach the capacity shell threshold \(\|\theta_{\text{MLP}}^{(t)}\|_F \leq r_{\max}\), we require:
\begin{equation}
	(1-\lambda)^{t - t_{\text{mem}}} \leq \left(\frac{r_{\max}}{\|\theta_{\text{mem}}\|_F}\right)^2 \implies t - t_{\text{mem}} \geq \frac{2}{\lambda} \log \left(\frac{\|\theta_{\text{mem}}\|_F}{r_{\max}}\right)
\end{equation}
This establishes the theoretical duration of the explaining-away plateau: \(\Delta t_{\text{plateau}} = \mathcal{O}(\frac{1}{\lambda} \log \frac{\|\theta_{\text{mem}}\|_F}{r_{\max}})\), mathematically mirroring the empirical delay laws established by \citet{Liu_Michaud_etal_2023}. Structural alignment (Phase 3) cannot reliably begin until this wait is over.

\subsection{Proof of KL Acceleration}

We now introduce the augmented loss: \(\mathcal{L}_{\text{total}} = \mathcal{L}_{\text{CE}} + \beta D_{\text{KL}}(\alpha^* \| \alpha)\). To prove that this bypasses Phase 2, we first derive the gradient of the structural penalty with respect to the logit \(\ell_j\).

\subsubsection{Self-Correcting Structural Gradient}
By definition, \(D_{\text{KL}}(\alpha^* \| \alpha) = \sum_i \alpha^*_i \log \alpha^*_i - \sum_i \alpha^*_i \log \alpha_i\). Taking the derivative with respect to \(\ell_j\):
\begin{equation}
	\frac{\partial}{\partial \ell_j} D_{\text{KL}}(\alpha^* \| \alpha) = -\sum_{i=1}^T \frac{\alpha^*_i}{\alpha_i} \frac{\partial \alpha_i}{\partial \ell_j}
\end{equation}
Using the standard Jacobian of the softmax function, \(\frac{\partial \alpha_i}{\partial \ell_j} = \alpha_i(\delta_{ij} - \alpha_j)\), we substitute this into the sum:
\begin{align}
	\frac{\partial}{\partial \ell_j} D_{\text{KL}}(\alpha^* \| \alpha) &= -\sum_{i=1}^T \frac{\alpha^*_i}{\alpha_i} \alpha_i(\delta_{ij} - \alpha_j) \\
	&= -\sum_{i=1}^T \alpha^*_i \delta_{ij} + \sum_{i=1}^T \alpha^*_i \alpha_j \\
	&= -\alpha^*_j + \alpha_j \sum_{i=1}^T \alpha^*_i
\end{align}
Because \(\alpha^*\) is a valid probability distribution, \(\sum_i \alpha^*_i = 1\). Thus, we obtain the beautifully simplified gradient:
\begin{equation}\label{eq:app_kl_grad}
	\frac{\partial D_{\text{KL}}}{\partial \ell_j} = \alpha_j - \alpha^*_j
\end{equation}
This gradient is strictly self-correcting: if the current attention \(\alpha_j\) is lower than the oracle \(\alpha^*_j\), the gradient is negative, pushing the logit \(\ell_j\) upwards. Crucially, this gradient is strictly independent of \(\mathcal{L}_{\text{CE}}\). It remains fully active even when the task gradient vanishes during Phase 2.

\subsubsection{Exponential Convergence of the KL Divergence}
To analyze the training trajectory, we track the continuous-time gradient flow of the KL divergence. By the chain rule:
\begin{equation}
	\frac{d}{dt} D_{\text{KL}}(\alpha^* \| \alpha) = \sum_{j=1}^T \frac{\partial D_{\text{KL}}}{\partial \ell_j} \frac{d\ell_j}{dt}
\end{equation}
Under gradient descent, the update to the interaction matrix \(M = W_Q^\top W_K\) descends the gradient of the loss. Abstracting the learning rate, inner product structure of the embeddings (\(e_q^\top \cdot e_j\)), and the chain rule into an effective penalty rate \(\tilde{\beta} \propto \eta \beta\), the flow on the logits simplifies geometrically to:
\begin{equation}
	\frac{d\ell_j}{dt} = - \tilde{\beta} \frac{\partial \mathcal{L}_{\text{total}}}{\partial \ell_j} = - \tilde{\beta} (\alpha_j - \alpha^*_j)
\end{equation}
Substituting this back into the divergence flow yields:
\begin{equation}\label{eq:app_kl_flow}
	\frac{d}{dt} D_{\text{KL}}(\alpha^* \| \alpha) = - \tilde{\beta} \sum_{j=1}^T (\alpha_j - \alpha^*_j)^2
\end{equation}
To convert the squared \(\ell_2\) gap on the right of Eq.~\ref{eq:app_kl_flow} into a fraction of \(D_{\text{KL}}\) itself, we appeal to a Polyak-Lojasiewicz (PL) inequality for the KL divergence on the simplex. Note that the gradient of \(D_{\text{KL}}(\alpha^* \| \alpha)\) with respect to the logits is exactly \(\nabla_\ell D_{\text{KL}} = \alpha - \alpha^*\) (Eq.~\ref{eq:app_kl_grad}); the dynamics in Eq.~\ref{eq:app_kl_flow} are thus already in the canonical form \(\dot{D}_{\text{KL}} = -\tilde{\beta} \|\nabla_\ell D_{\text{KL}}\|_2^2\), and exponential convergence follows from any PL bound \(\|\alpha - \alpha^*\|_2^2 \geq 2\mu \, D_{\text{KL}}\) with \(\mu > 0\).
\begin{lemma}[Local PL for KL]\label{lem:local_pl}
	Let \(\alpha^* \in \Delta^{T-1}\) have full support, and write \(\alpha = \alpha^* + \Delta\). A second-order Taylor expansion of \(D_{\text{KL}}(\alpha^* \| \cdot)\) about \(\alpha^*\) gives
	\begin{equation}\label{eq:app_kl_taylor}
		D_{\text{KL}}(\alpha^* \| \alpha) \;=\; \tfrac{1}{2} \sum_{j=1}^T \frac{\Delta_j^2}{\alpha_j^*} + \mathcal{O}(\|\Delta\|_2^3),
	\end{equation}
	while \(\|\alpha - \alpha^*\|_2^2 = \sum_j \Delta_j^2\). Comparing the two ratios,
	\begin{equation}\label{eq:app_pl}
		\frac{\|\alpha - \alpha^*\|_2^2}{D_{\text{KL}}(\alpha^* \| \alpha)} \;\geq\; 2 \min_j \alpha_j^* \;=:\; 2\mu,
	\end{equation}
	so \(\|\alpha - \alpha^*\|_2^2 \geq 2\mu \, D_{\text{KL}}(\alpha^* \| \alpha)\) for \(\alpha\) in a neighborhood of \(\alpha^*\), with \(\mu > 0\) whenever \(\alpha^*\) has full support.
\end{lemma}
For sparse oracles (\(\alpha^*_j = 0\) at distractor positions), Eq.~\ref{eq:app_pl} restricts to \(\mathrm{supp}(\alpha^*)\); residual mass at distractors decays separately under the same flow, since \(\dot{\ell}_j = -\tilde{\beta}\alpha_j\) drives \(\alpha_j \to 0\) at any \(j \notin \mathrm{supp}(\alpha^*)\). Substituting the PL bound into Eq.~\ref{eq:app_kl_flow}:
\begin{equation}
	\frac{d}{dt} D_{\text{KL}}(\alpha^* \| \alpha) \;\leq\; -2\tilde{\beta}\mu \cdot D_{\text{KL}}(\alpha^* \| \alpha).
\end{equation}
This ordinary differential equation yields a strict exponential decay. Letting \(D_t = D_{\text{KL}}(\alpha^* \| \alpha^{(t)})\), and absorbing the constant \(2\mu\) into \(\tilde{\beta}\), we have:
\begin{equation}
	D_t \leq D_0 e^{-\tilde{\beta} t}
\end{equation}
We note that the PL bound in Lemma~\ref{lem:local_pl} is derived locally about \(\alpha^*\); it certifies the asymptotic exponential rate once \(\alpha\) enters a neighborhood of the oracle. Globally, monotone descent of \(D_{\text{KL}}\) along Eq.~\ref{eq:app_kl_flow} is unconditional (the right-hand side is non-positive), so the iterates necessarily reach the local PL regime in finite time.

\subsubsection{Bounding the Grokking Delay}
Generalization requires crossing the bottleneck threshold \(D_t \leq \delta_{\text{nec}}\) (Corollary~\ref{cor:kl_necessity}). Solving the exponential bound for the required time \(t\):
\begin{equation}
	\delta_{\text{nec}} = D_0 e^{-\tilde{\beta} \Delta t_{\text{grok}}^{\text{KL}}} \implies \Delta t_{\text{grok}}^{\text{KL}} \leq \frac{1}{\tilde{\beta}} \log\left(\frac{D_0}{\delta_{\text{nec}}}\right)
\end{equation}
This completes the proof. The KL penalty entirely bypasses the weight decay wait-time, reducing the grokking delay to a continuous trajectory that scales inversely linearly with the effective penalty strength \(\tilde{\beta}\).

\section{Extensions and Theoretical Conjectures}\label{app:extensions}

In the main text, our theoretical results (Theorem~\ref{thm:decoupling} and Theorem~\ref{thm:kl_acceleration}) are formally derived for a single-head, linear ReLU attention mechanism. This is a standard abstraction in Transformer theory \citep{Ahn_Cheng_etal_2024, Zhang_Frei_etal_2023} that isolates the convex-combination routing dynamics while preserving the analytical tractability of the Jacobian. Here, we formalize how these principles extend to full, standard Transformer architectures.

\subsection{Extension to Softmax Architecture}\label{app:softmax}

For standard softmax configurations where \(\alpha_j = \exp(\ell_j)/\sum_i \exp(\ell_i)\), the structural dynamics operate identically to our ReLU derivations under one key mathematical difference: the intrinsic shift-invariance of the softmax transformation. Because softmax is invariant to additive shifts, weight decay on \(W_Q\) and \(W_K\) does not automatically push predictions to uniform distributions merely by collapsing the raw logit values below zero.

However, we conjecture that the explaining-away plateau (Phase 2) and the subsequent convergence dynamics hold identically for softmax attention, driven by the contraction of the \emph{logit variance}.

\begin{conjecture}[Softmax Logit Variance Contraction]\label{conj:softmax}
	Under a norm-contractive optimizer (Definition~\ref{def:nco}), weight decay on \(W_Q\) and \(W_K\) minimizes the Frobenius norm of the bilinear interaction matrix \(M = W_Q^\top W_K\). Because the pre-attention logits are given by \(\ell_j = \frac{1}{\sqrt{d}} e_q^\top M e_j\), the variance of the logits across the context window is strictly bounded by \(\|M\|_F\).

	Consequently, during the explaining-away plateau (\(\mathcal{L}_{\text{CE}} \to 0\)), the task-driven source term vanishes, and weight decay multiplicatively contracts the centered logit variance \(\mathrm{Var}_j[\ell_j - \bar{\ell}]\). As this variance contracts toward zero, the softmax distribution inherently smooths toward the uniform distribution \(\alpha_j \to 1/T\), initiating Phase 3.
\end{conjecture}

Empirically, our standard network evaluations in Section~\ref{sec:experiments} (which utilize standard softmax attention) map identically to the structural trajectories predicted by this geometric contraction, validating the conjecture as observationally complete. Furthermore, the gradient of our KL acceleration penalty (Eq.~\ref{eq:app_kl_grad}) is exact for softmax, meaning the \(\mathcal{O}(1/\tilde{\beta})\) theoretical speedup natively applies without modification.

\subsection{Multi-Head Distributed Bayesian Condition}\label{app:multihead}

We presented the core Convergence and Decoupling theorems over single-head frameworks to establish a strict, isolatable representation bottleneck. In full architectures, the final representation is an aggregation across \(H\) distinct attention heads \(\{\alpha^{(r)}\}_{r=1}^H\):
\begin{equation}
	h(s) = \sum_{r=1}^H W_O^{(r)} \sum_{j=0}^{T-1} \alpha_j^{(r)} W_V^{(r)} e_j(s)
\end{equation}
Rather than enforcing that any individual head \(\alpha^{(r)}\) must uniformly map all informative coordinates internally, the representation gradient distributes across the heads, allowing them to specialize in fractional vector overlaps.

\begin{conjecture}[Distributed Bayesian Condition]\label{conj:multihead}
	For an \(H\)-head attention layer, the strict per-head Bayesian Condition \(\mathcal{B}_\gamma\) relaxes into a distributed sum. There exists a pooled threshold \(\gamma^*_{\text{MH}}\) such that perfect generalization is possible if the aggregated informative mass satisfies:
	\begin{equation}
		\sum_{r=1}^H \alpha_j^{(r)} \geq \gamma^*_{\text{MH}} \quad \forall j \in \mathcal{I}^*
	\end{equation}
	where \(\gamma^*_{\text{MH}}\) is smaller than the single-head threshold \(\gamma\) by a factor depending on the output projection geometry (\(W_O\)) and the degree of orthogonal head redundancy.
\end{conjecture}

Under this distributed condition, the network does not require a single winning ticket in a single head, but rather an ensemble of heads whose combined posterior mass sufficiently covers the oracle dependency graph. Extending our max-margin convergence guarantees to this distributed multi-head routing space is a promising direction for future theoretical work.

\section{Additional Empirical Results}\label{app:additional_experiments}

The main text uses modular addition to isolate the mechanism in its cleanest form. Here we report the supporting diagnostics used to check that the phenomenon is not an artifact of one plot, one regularizer, or one optimizer trace.

\subsection{The Attention Gradient Really Vanishes}

\begin{figure}[t]
	\centering
	\includegraphics[width=0.95\textwidth]{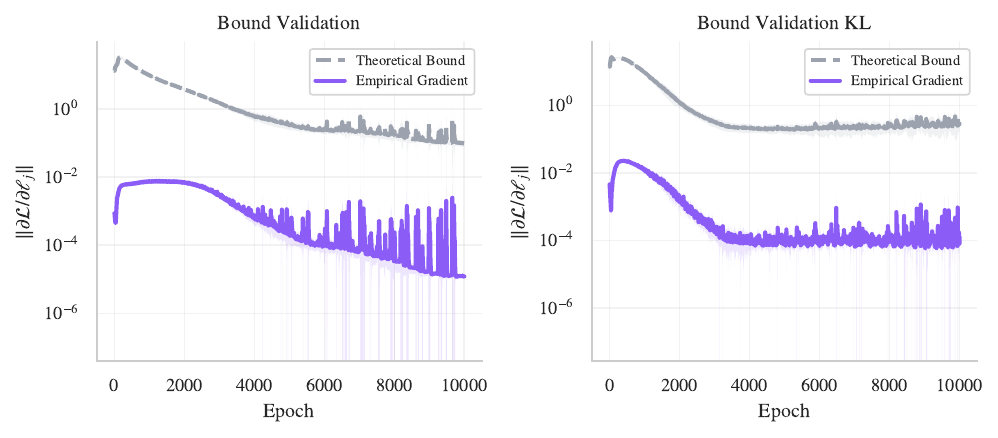}
	\caption{\textbf{Explaining-away bound.} The empirical attention-logit gradient remains below the theoretical bound from Lemma~\ref{lem:explaining_away} in both baseline and KL-regularized training. The bound is conservative, but the qualitative statement is sharp: once cross-entropy is small, the task gradient into attention is tiny.}
	\label{fig:app_explaining_away_bound}
\end{figure}

Figure~\ref{fig:app_explaining_away_bound} verifies the key analytic inequality behind the plateau. The measured gradient is far below the upper bound, especially after memorization. This is the empirical signature of explaining away: the MLP has already solved the training labels, so the attention layer no longer receives a useful task-driven correction.

\subsection{Structural Priors, Sparse Priors, and Distillation}

\begin{figure}[t]
	\centering
	\includegraphics[width=0.48\textwidth]{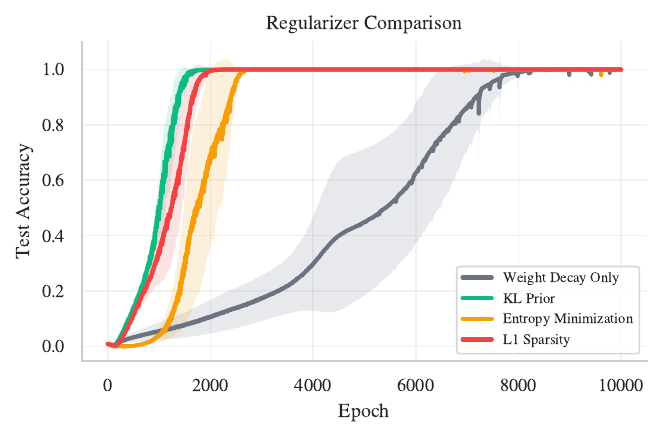}
	\includegraphics[width=0.48\textwidth]{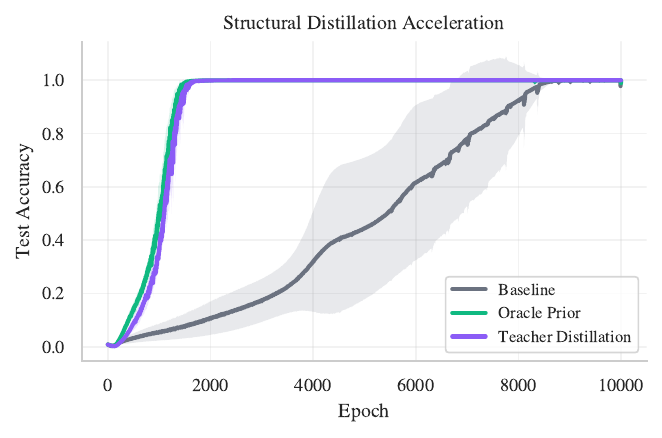}
	\caption{\textbf{Structural priors.} \textbf{(Left)} Generic sparsity pressures help, but the KL prior reaches the generalizing solution fastest because it specifies where sparse mass should go. \textbf{(Right)} A learned teacher attention map transfers nearly the same acceleration as the oracle prior.}
	\label{fig:app_structural_priors}
\end{figure}

\begin{figure}[t]
	\centering
	\includegraphics[width=0.48\textwidth]{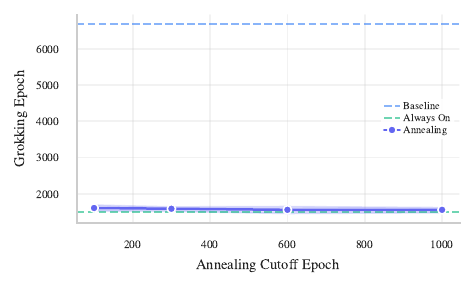}
	\includegraphics[width=0.48\textwidth]{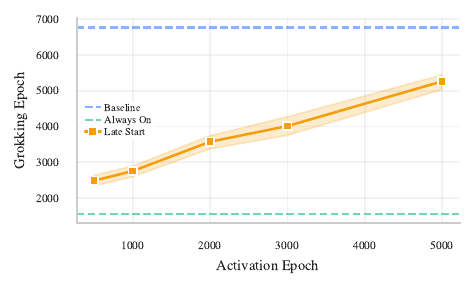}
	\caption{\textbf{Timing the structural intervention.} \textbf{(Left)} A short early KL intervention almost matches an always-on prior, suggesting that the routing ticket persists after the prior is removed. \textbf{(Right)} Late injection still accelerates grokking, but the transition moves later with the activation epoch.}
	\label{fig:app_kl_timing}
\end{figure}

These ablations clarify what the KL intervention is doing. Entropy and \(\ell_1\) penalties accelerate grokking relative to weight decay alone, so sparsity itself is useful. The KL prior is stronger because it is directed: it pushes probability mass toward the informative support rather than merely making the posterior sharper. Figure~\ref{fig:app_structural_priors} also shows that exact oracle access is not essential for the mechanism. A teacher-derived structural prior transfers much of the same acceleration, making oracle KL best understood as an analytical probe and an upper bound on practical attention supervision.

\subsection{Routing Beyond the Fixed Modular-Addition Setup}

\begin{figure}[t]
	\centering
	\includegraphics[width=0.48\textwidth]{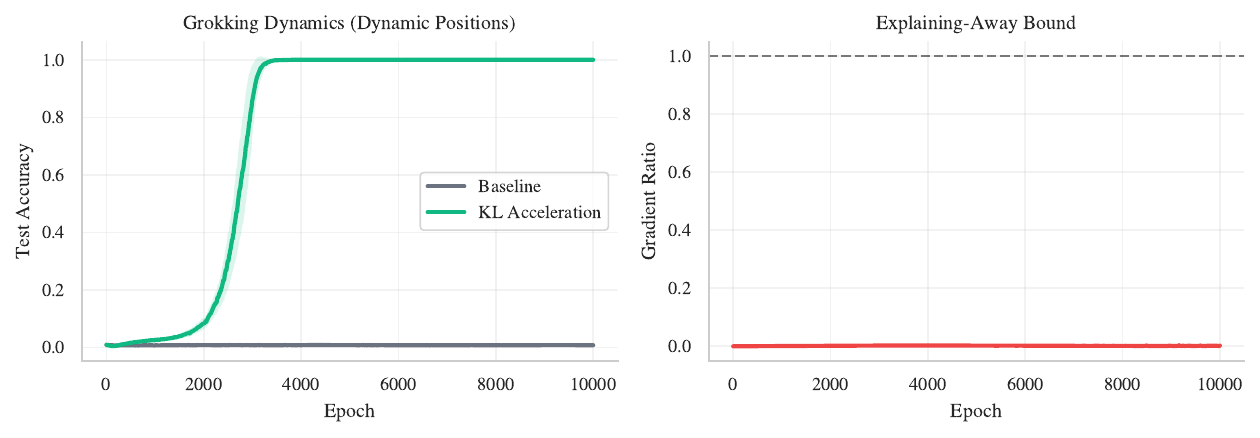}
	\includegraphics[width=0.48\textwidth]{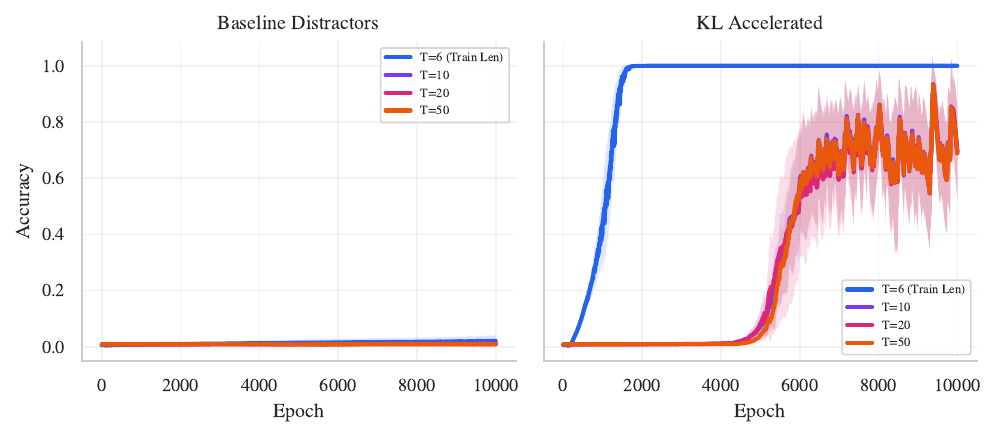}
	\caption{\textbf{Routing under distractors.} \textbf{(Left)} With informative positions randomized per sequence, a sequence-dependent structural prior \(\alpha^*(s)\) still bypasses the plateau. \textbf{(Right)} Under added distractors, KL-trained models extrapolate better to longer contexts, although the longest lengths remain imperfect.}
	\label{fig:app_dynamic_length}
\end{figure}

\begin{figure}[t]
	\centering
	\includegraphics[width=0.48\textwidth]{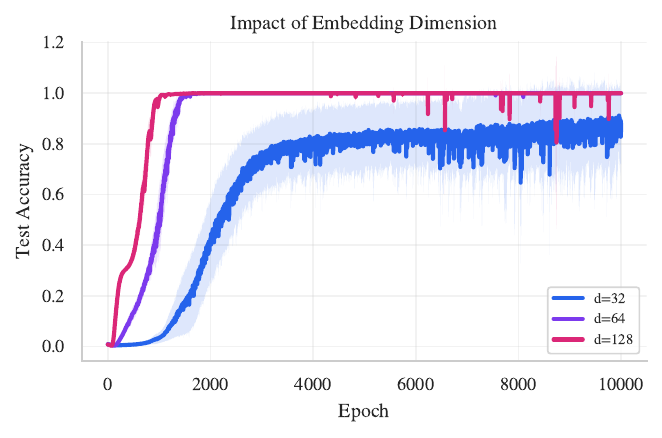}
	\includegraphics[width=0.48\textwidth]{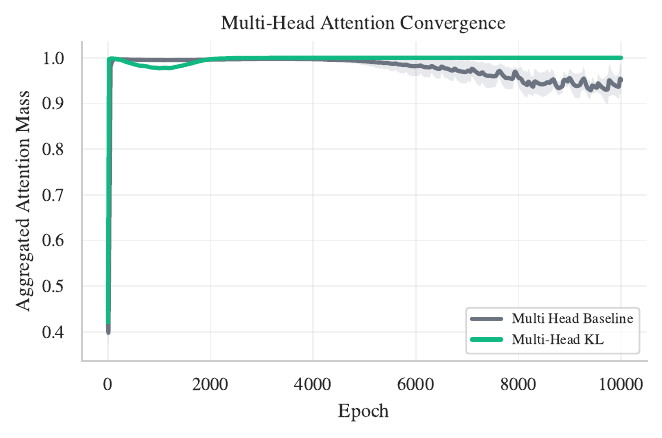}
	\caption{\textbf{Geometry and distributed attention.} \textbf{(Left)} Larger embedding dimension accelerates generalization, consistent with the geometric role of representation separation and subspace incoherence. \textbf{(Right)} In a 4-head model, aggregate oracle attention mass rapidly approaches one; KL keeps this distributed mass more tightly aligned.}
	\label{fig:app_geometry_multihead}
\end{figure}

The fixed-query, fixed-position analysis is deliberately narrow. Figure~\ref{fig:app_dynamic_length} tests two ways to loosen it. When the informative locations vary across examples, a sequence-dependent oracle \(\alpha^*(s)\) restores the same structural bypass. When the number of distractors grows at test time, the KL-trained model preserves substantially better accuracy than the baseline, though extrapolation to the longest contexts remains a robustness result rather than a perfect-generalization claim. Figure~\ref{fig:app_geometry_multihead} gives complementary evidence for the geometric and multi-head extensions: higher-dimensional embeddings learn the structural solution faster, and multi-head models can satisfy the routing condition through aggregate mass rather than a single monolithic head.

\clearpage

\subsection{Task and Optimizer Robustness}

\begin{figure}[t]
	\centering
	\includegraphics[width=0.98\textwidth]{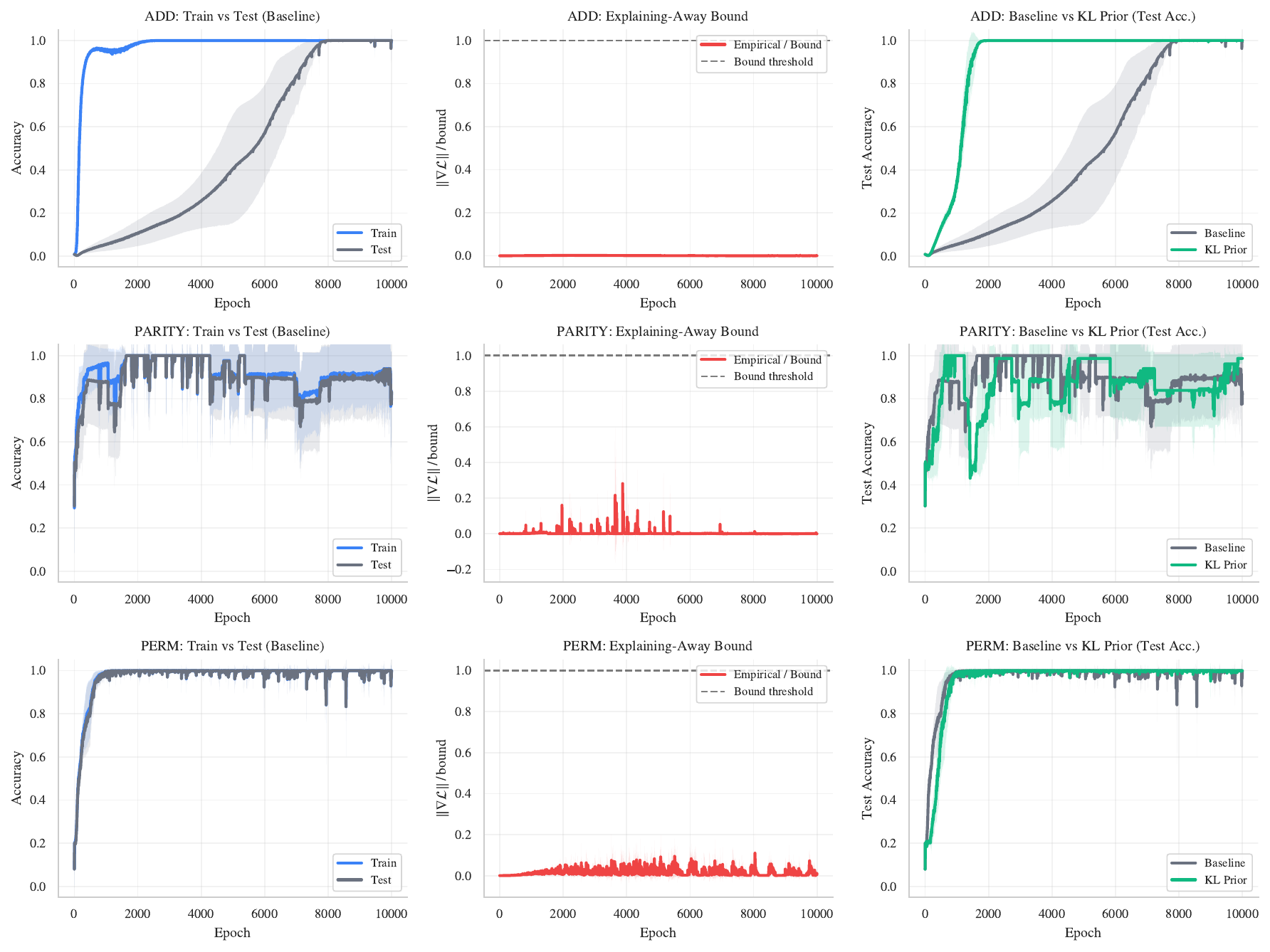}
	\caption{\textbf{Task diversity.} Modular addition gives the cleanest delayed transition, while sparse parity and permutation composition show noisier or faster transitions. Across all three tasks, the attention-gradient diagnostic remains bounded, and KL accelerates the structural route when a plateau is present.}
	\label{fig:app_task_diversity}
\end{figure}

\begin{figure}[t]
	\centering
	\includegraphics[width=0.48\textwidth]{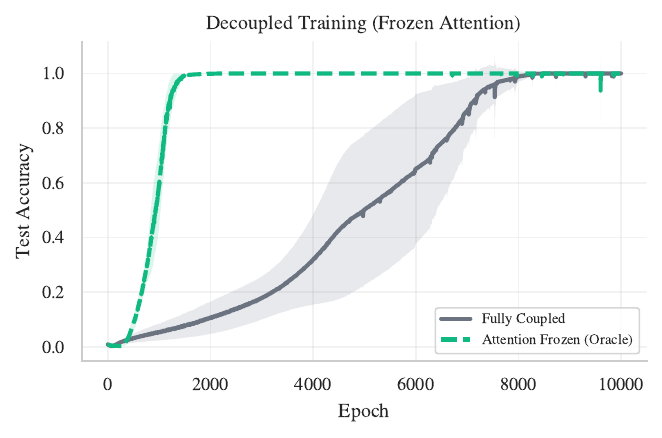}
	\includegraphics[width=0.48\textwidth]{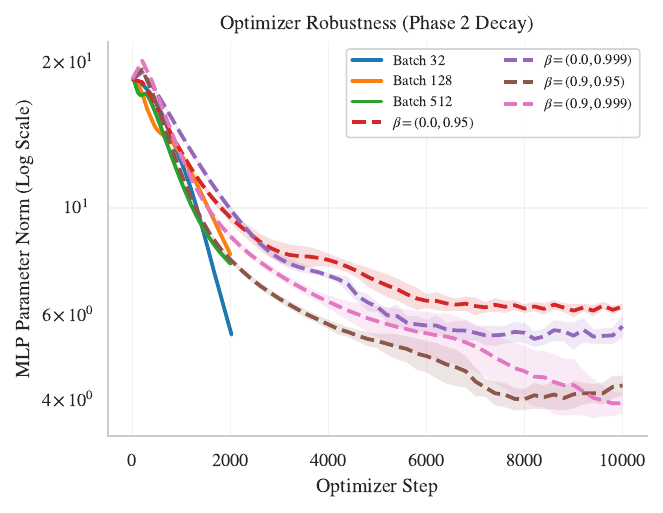}
	\caption{\textbf{Decoupled and optimizer controls.} \textbf{(Left)} Once attention is frozen to oracle routing, the downstream network generalizes far earlier than the fully coupled baseline. \textbf{(Right)} MLP norm decay persists across batch-size and Adam momentum variants, supporting the qualitative norm-contraction assumption.}
	\label{fig:app_decoupled_optimizer}
\end{figure}

The task grid in Figure~\ref{fig:app_task_diversity} is not meant to claim that every algebraic task has identical phase geometry. It shows the narrower point needed here: the attention-gradient starvation diagnostic remains compatible with the bound, and the KL structural prior accelerates the route when routing is the bottleneck. Figure~\ref{fig:app_decoupled_optimizer} then isolates the two assumptions that remain implicit in the theorem. Freezing oracle attention removes the coupled-routing bottleneck, and the remaining MLP generalizes quickly. Meanwhile, optimizer perturbations change rates and floors, but preserve the broad norm-contractive decay used in the convergence argument.

\section{Experimental Details and Reproducibility}\label{app:reproducibility}

Code for all experiments is provided in the repository. Here, we will provide the full configuration and hyperparameters used to generate the figures in the main text and appendix.

\subsection{Hardware and Software}

Experiments were executed with a single NVIDIA RTX PRO 6000 Blackwell Server Edition GPU. The code is implemented in PyTorch 2.5.1 and Python 3.12. Training was performed using Automatic Mixed Precision (AMP) and \texttt{torch.compile} for optimized performance. Full-batch training was used for all experiments to minimize stochasticity, with the exception of the Autoregressive Grokking experiments which used a batch size of 512.

\subsection{Model Architecture and Hyperparameters}

Unless otherwise specified, we used a single-layer Transformer with a single attention head and a two-layer ReLU MLP. The default hyperparameters are summarized in Table~\ref{tab:hyperparams}.

\begin{center}
	\small
	\begin{tabular}{llc}
		\toprule
		Category & Hyperparameter & Value \\
		\midrule
		\textbf{Architecture} & Number of Layers & 1 \\
		& Number of Heads (\(H\)) & 1 (4 for multi-head ablation) \\
		& Model Dimension (\(d_{\text{model}}\)) & 64 \\
		& MLP Dimension (\(d_{\text{mlp}}\)) & 256 (ratio 4) \\
		& Activation Function & ReLU \\
		& Vocabulary Size & \(p + 3\) (for modular addition) \\
		\midrule
		\textbf{Optimization} & Optimizer & AdamW \\
		& Learning Rate (\(\eta\)) & \(10^{-3}\) \\
		& Adam \(\beta_1, \beta_2\) & 0.9, 0.999 \\
		& Weight Decay (\(\lambda\)) & 1.0 \\
		& Warmup Epochs & 10 \\
		& Max Gradient Norm & 1.0 \\
		& Training Epochs & 10,000 \\
		& Batch Size & Full Batch \\
		\midrule
		\textbf{Data} & Modulus (\(p\)) & 113 \\
		& Sequence Length (\(T\)) & 5 (for modular addition) \\
		& Train Fraction & 0.3 \\
		& Number of Random Seeds & 5 (0--4) \\
		\bottomrule
	\end{tabular}
	\label{tab:hyperparams}
\end{center}

\subsection{Experiment-Specific Configurations}

Below we detail the specific variations and sweep parameters for the empirical results presented in the paper.

\paragraph{Four-Phase Dynamics (Figure~\ref{fig:four_phases}).} Trained using the baseline configuration (no KL prior, \(\beta_{\text{KL}} = 0\)). The structural task gradient norm was tracked using the Frobenius norm of the gradients with respect to the attention interaction matrix \(M\).

\paragraph{Isolating Norm from Structure (Figure~\ref{fig:decoupling}).} A factorial sweep was performed across two dimensions:
\begin{itemize}[noitemsep,topsep=0pt]
	\item \textbf{Norm Condition:} Weight Decay (\(\lambda=1.0\)) vs. No Weight Decay (\(\lambda=0.0\)).
	\item \textbf{Structural Condition:} Baseline (no prior), Oracle Prior (\(\alpha = \alpha^*\)), and Adversarial Prior (attention forced onto distractors).
\end{itemize}
Norm-matched controls used a targeted \(W_Q / W_K\) L2 penalty (\texttt{qk\_penalty}=0.5) to isolate the structural effect from weight-norm changes.

\paragraph{KL Acceleration and Scaling (Figure~\ref{fig:kl_scaling}).} The KL penalty strength \(\beta_{\text{KL}}\) was swept across the values: \(\{0.0, 0.02, 0.025, 0.033, 0.05, 0.1, 0.2, 0.5, 1.0, 5.0\}\). The grokking delay \(\Delta t_{\text{grok}}\) was defined as the first epoch where test accuracy reached \(99\%\).

\paragraph{Bayesian Ticket vs. Lottery Ticket (Figure~\ref{fig:lottery_ticket}).}
\begin{itemize}[noitemsep,topsep=0pt]
	\item \textbf{Lottery Ticket:} Models were trained for 10k epochs, pruned to the sparse structural mask, and re-initialized with the original weights.
	\item \textbf{Bayesian Ticket:} Models were initialized with fresh random weights and regularized with a constant KL prior (\(\beta_{\text{KL}}=1.0\)).
\end{itemize}

\paragraph{Appendix Experiments:}
\begin{itemize}[leftmargin=*]
	\item \textbf{Explaining-Away Bound:} Compared the empirical attention-logit gradient against \(2L_f\|W_V\|\sqrt{2\mathcal{L}_{CE}}\) throughout baseline and KL-regularized training.
	\item \textbf{Structural Priors:} Compared the oracle KL prior against entropy minimization, \(\ell_1\) attention sparsity, teacher-student structural distillation, KL annealing, and late-start KL injection.
	\item \textbf{Task Diversity:} Validated across Modular Addition (\(p=113\)), Sparse Parity (\(n=20, k=3\)), and Permutation Composition (\(S_5\)).
	\item \textbf{Autoregressive Grokking:} Next-token prediction on \(k\)-sparse parity with \(n=20\) and \(k=3\). Used a batch size of 512.
	\item \textbf{Multi-Head Dynamics:} 4-head attention where the informative signal was distributed across heads. Evaluated using a query index of 4 and informative positions \([1, 3]\).
	\item \textbf{Dimension Ablation:} Varied \(d_{\text{model}} \in \{32, 64, 128\}\) to observe subspace incoherence effects.
	\item \textbf{Length Generalization:} Trained on sequence length 6 and evaluated OOD on lengths \(\{10, 20, 50\}\).
	\item \textbf{Dynamic Routing:} Evaluated modular addition where informative positions were randomized per sequence.
	\item \textbf{Optimizer Robustness:} Swept batch size and Adam momentum parameters to test whether the Phase 2 MLP norm decay survives stochastic and optimizer-level perturbations.
\end{itemize}

\end{document}